\documentclass[12pt]{amsart}
\usepackage{graphicx}
\usepackage{xspace}
\usepackage[margin=1in]{geometry}
\usepackage{amssymb}
\usepackage{amsmath}
\usepackage{amsthm}
\usepackage{thmtools}
\usepackage{color}
\usepackage{xcolor}
\usepackage{hyperref}
\newtheorem{theorem}{Theorem}
\newtheorem{corollary}{Corollary}

\newtheorem{lemma}{Lemma}

\theoremstyle{definition}

\DeclareMathOperator*{\argmin}{\arg\min}
\begin{document}
\title[Divide and Predict]{Divide and Predict: An Architecture for
  Input Space Partitioning and Enhanced Accuracy}
\author{Fenix W. Huang}
\address{Biocomplexity Institute, University of Virginia}
\email{fwh3zc@virginia.edu}
\author{Henning S. Mortveit}
\address{Department of Systems and Information Engineering, University of Virginia}
\email{henning.mortveit@virginia.edu}
\author{Christian M. Reidys}
\address{Department of Mathematics, University of Virginia}
\email{cmr3hk@virginia.edu}
\keywords{Supervised learning, training data stratification, influence, heterogeneity}
\begin{abstract}
In this article the authors develop an intrinsic measure for
quantifying heterogeneity in training data for supervised
learning. This measure is the variance of a random variable which
factors through the influences of pairs of training points.  The
variance is shown to capture data heterogeneity and can thus be used
to assess if a sample is a mixture of distributions.  The authors
prove that the data itself contains key information that supports a
partitioning into blocks.  Several proof of concept studies are
provided that quantify the connection between variance and
heterogeneity for EMNIST image data and synthetic data.  The authors
establish that variance is maximal for equal mixes of distributions,
and detail how variance-based data purification followed by
conventional training over blocks can lead to significant increases in
test accuracy.
\end{abstract}
\maketitle

\section{Introduction}
\label{S:introduction}

The pursuit of advances in machine learning and generative AI has led
to a massive increase in computing requirements.  Complex data
containing multiple distributions has led to more data-hungry and
advanced architectures, running on increasingly larger data centers
with power demands equivalent to those of mid-size cities, have
definitely led to breakthroughs. What about advances in the underlying
theory?

This paper introduces a novel measure that captures
\emph{heterogeneity} of mixed distributions. The measure offers an
intrinsic way to quantify data complexity. Furthermore, we prove that
there exists an algorithm that allows one to ``untangle'' the data.
Our theory thus supports untangling data which in turn enables the use
of simpler architectures that maintain high accuracy, but have a
significantly smaller energy footprint.

Our measure is based on the notion of influence introduced in the
1980s (see, e.g.,~\cite{Cook-Weisberg:1980}) and was motivated by the
high cost of model retraining. Influence is genuinely local in the
sense that it is associated with specific pairs of data points and was
used to quantify how model parameters for local loss contributions
change as we perturb a training point by an infinitesimal amount. This
led to a data-calculus, capable of deriving closed form expressions of
training data perturbations.  The topic was revisited
in~\cite{Koh-Liang:2018} in the context of machine learning, where the
authors showed how influence can be used to expose vulnerability of
training data, as well as how to easily manipulate training data to
force counterintuitive results, making the case that influence
functions represent a useful instrument to understand black-box
predictions.  However, robust statistics shows that when the influence
function is unbounded, even an infinitesimal perturbation of the data
can induce arbitrarily large changes in the estimator, particularly
when the assumed model fails to adequately represent the observed data
distribution \cite{hampel1986robust}.

The key contribution of this paper is to treat influence as a global
measure of a data set instead of as a local measure at pairs of data
points, and to demonstrate the feasibility of using this concept to
partition training data into blocks on which existing machine learning
methods exhibit higher test accuracy.

\medskip

Supervised learning takes a set of the form
\begin{equation}
\label{eq:sample}
Z = \{z_i = (x_i, y_i) \mid
x_i \in E, y_i \in \text{$\mathbb{R}$  for  $1 \le i \le n$}\}\;,
\end{equation}
where $E$ is some suitable space (e.g., images or $\mathbb{R}^k$), a
parametrized model class $\mathcal{F}_\theta$ with $\theta \in
\mathbb{R}^m$, a real-valued \emph{loss function} $\mathcal{L}_Z$, and
seeks to determine
\begin{equation}
\hat{\theta} = \argmin_{\theta} \mathcal{L}_Z(F_\theta) \;,
\end{equation}
where the primary objective is a model with high prediction quality.
There are now a range of advanced, computationally demanding
architectures to tackle this problem. Examples include deep neural
networks~\cite{LeCun:15}, Mixture of Experts~\cite{Mu:25}, and, more
recently, transformers with self-attention
mechanisms~\cite{Vaswani:17,Lin:22,Geshkovski:25}.
A common assumption is that the \emph{training set} $Z$ is a sample
from some input-output process that can be captured well by a single
statistical distribution $p(y|x,\hat{\theta})$. However, even advanced
models can struggle when this is not the case.

As for quality of prediction, the convergence framework of
sufficiently over-parameterized networks~\cite{Zhu-Li-Song:2019}
guarantees close to zero training loss for any non-degenerate data set
at a global minimum.
When generalization falls short of expectations, a common strategy is
to employ more complex models or architectures. However, recent
studies on learning from heterogeneous or mixture data indicate that a
single global model fails to recover the individual components, and
merely increasing model capacity does not resolve the errors induced
by such heterogeneity~\cite{JacobsJordanNowlanHinton1991,
  mannering2016unobserved, vardhan2025learning}.  Alternatively,
various extrinsic measures, such as the use of domain experts, may be
employed to organize data or remove ``outliers''.
A Variational Autoencoder (VAE) is a framework that maps inputs into a
latent space, enabling effective representation learning,
dimensionality reduction, and modeling of complex data distributions
\cite{KingmaWelling2013}. However, standard VAE assumes a unimodal
latent prior and therefore cannot reliably separate distinct mixtures
of distributions in the data. When the latent representations of
different sub-populations overlap, the decoder will produce outputs
biased toward a global average rather than capturing each distribution
individually~\cite{LavdaGregorovaKalousis2020, SinghOgunfunmi2022}.

The notion of \emph{mixture models} emerged in the statistical
literature to address the situation where this assumption fails. This
work is largely concerned with the case where $Z$ arises through $K$
(a hyperparameter) identical distributions (e.g., Gaussians) albeit
with different parameters. The work has been extended to mixtures of
non-identical distributions, see~\cite{Pal:24}.

A \emph{mixture of experts} (MoE) is a machine learning architecture
with multiple specialized sub-models referred to as \emph{experts} and
a gating network routing input to the subset of the sub-models. It was
introduced in~\cite{Eigen:14,Mu:25,Chen:22,Shazeer:17,Baldacchino:16}
based on the fundamental assumption that the input features contain
sufficient signal to allow the gating network (router) to distinguish
between the underlying distributions. The presence of such a signal is
essential~\cite{JacobsJordanNowlanHinton1991,
  hennig1996identifiability}: suppose we have data generated by two
different functions and no further information. Then the MoE
encounters a one-to-many mapping problem and the router has no basis
to partition the data. As a result, gating weights converge towards an
equal distribution. That is, the model will minimize the global loss
by predicting an average function, rather than recovering the
individual functions, a shortcoming induced by the heterogeneity in
the latent labels rather than the observed features.

\medskip

Our approach is motivated by prior work on multiple sequence arrays of
viral RNA sequences~\cite{Bura:24}. Here data-heterogeneity manifests
through a mixture of viral variants in the data, and the objective is
to partition the data into these variants. Such determinations are
typically accomplished with the help of experts and
experiments. In~\cite{Bura:24} a different method is introduced which
is based on studying the change of the variance of information between
two columns of the array. It is shown that there always exist
sequences whose removal lower the variance of information leading to a
procedure by which blocks of sequences belonging to a single variant
can be identified. The key here is that a variant corresponds to a set
of molecular bonds (motif) which in turn relates certain
columns. Lowering the variance of information is then achieved by
removing sequences that are not consistent with these relations.
In the context of machine learning, choices of loss functions and
architectures can be seen as proxies for these motives. It is
therefore natural to consider utilizing the data-removal induced
change of information fabric to identify sub-samples of data that are
``homogeneous''.

\medskip

Let $Z$ represent a mix of distributions without making assumptions
about their nature or numbers. The results of this paper suggest
considering a two-stage approach to supervised learning: a
stratification step where $Z$ is partitioned into sets $Z_1$ through
$Z_k$, followed by a training step for individual sub-models over the
various $Z_i$'s. Subsequent prediction uses a classifier to route
input to the appropriate sub-model which is then applied, see
Figure~\ref{fig:concept}.
\begin{figure}[ht]
\centerline{\includegraphics[width=0.85\linewidth]{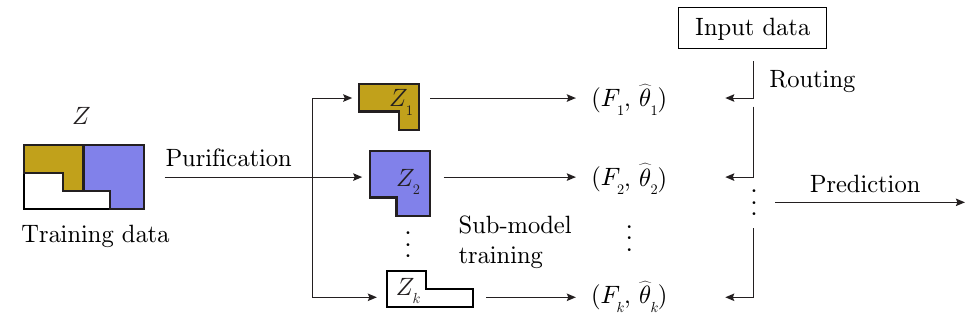}}
\caption{The two-stage approach involving purification followed by
  conventional machine learning. Prediction is done using a classifier
  routing input to the appropriate sub-model.}
\label{fig:concept}
\end{figure}

To describe our approach we fix a model class $\mathcal{F}_\theta$,
loss function $f_Z(\theta) = \sum_z L(z,\theta)$, and introduce the
augmented function $u_Z(\theta) = f_Z(\theta) + \sum_{z} \epsilon_z
L(z,\theta)$. In \cite{Koh-Liang:2018}, using the framework of
\cite{Cook-Weisberg:1980}, it is shown that for~$z,z'\in Z$, the
derivative $\frac{\partial}{\partial \epsilon_z}L(z',\theta)$
quantifies the influence of the scaling of $z$ on the loss at $z'$,
and that this derivative can be positive as well as negative, even
when $Z$ consists of a single distribution.

To extend local influence to a global concept of data set influence,
we introduce a random variable~$X$ defined over pairs $\{z,z'\}\subset
Z$ by $X(\{z,z'\})= \left[\frac{\partial}{\partial \epsilon_z}
  L(z',\hat\theta)\right]$. The properties of $X$ encapsulate features
of the entire data set. Clearly,~$X$ depends on the model class, loss
function as well as the model $\hat{\theta}$.  Once global data
properties are framed as a random variable, a natural next step is to
consider its moments~$\mathbb{E}[X^\ell]$.
Intuitively, the variance should be minimized to obtain blocks $Z_i$
following a single distribution. In order to stratify the set $Z$ into
blocks of minimal variance the key question is therefore if we can
iteratively remove subsets $M$ to minimize the variance, thereby
breaking up~$Z$ into ``homogeneous'' pieces.

Our main results are Theorems~\ref{T:manyk} and~\ref{T:variance}
which, under mild assumptions on the size of $Z$, but also convexity,
guarantee the existence of a set~$M$ whose removal lowers the variance
as well as the even-ordered raw moments of~$X$.
The implication of these results is that stratification of $Z$ and
subsequent training of sub-models can increase test accuracy relative
to the direct approach of conventional training on~$Z$. Our
proof-of-concept experiments indicate that the increase in test
accuracy can be quite significant, even when the convexity assumption
fails to hold.

We emphasize that the blocks generated through the stratification
depend on the underlying loss function and model class
$\mathcal{F}_\theta$. Stratification mimics what happens in
experimental physics where parameters and methods are set and specific
answers are derived accordingly. The choice of model class and loss
function are tantamount to the choice of experimental setup.
We remark that the stratification itself constitutes a window into the
``black box'' of the learning process itself. Here, using a
sufficiently expressive model (which exist by~\cite{Zhu-Li-Song:2019})
to probe $Z$ through moments and variance allows one to identify the
presence of multiple distributions within the input data.

\medskip

Suppose the data follows a single, unique distribution and that we
have a model, $\hat\theta$, with a sufficiently small loss. In this
scenario the influence for any pair $z\neq z'$ is arguably minimal,
i.e.,~$\frac{\partial}{\partial\epsilon_z}L(z',\theta)$ has a small
absolute value and fluctuations are induced by the varying closeness
of the particular $z,z'$ to the model.  In case the input data
contains multiple distributions two effects manifest: first the
quality of the model in terms of prediction can only decrease, and
secondly the absolute values of the derivatives become larger and
larger--even though the loss function is still small as guaranteed by
\cite{Zhu-Li-Song:2019}.  It is informative to consider the first
effect taken to its extreme, namely an input data set that is entirely
random. In this scenario influence measurements are of limited value
since they are relative to the loss function and network
employed. Random data hinders any inferences obtained from employing
these.

\medskip

The paper is organized as follows: In Section~\ref{S:basic} we
establish terminology and some basic facts (Lemma~\ref{L:1}
through~\ref{L:basic}) for how to express influence via the Hessian
\cite{Koh-Liang:2018,Cook-Weisberg:1980}. In Section~\ref{sec:main} we
cover Theorem~\ref{T:manyk}, Theorem~\ref{T:variance}, and
Corollary~\ref{C:1} which, under the assumption of convexity,
establish the fact that there are data points whose removal will
reduce the variance of $X$. A proof outline is provided for
Theorem~\ref{T:manyk}; the full proof is given in
Section~\ref{sec:proofs} along with proofs of all supporting results.
Section~\ref{S:appl} provides proof-of-concept applications showing
how variance captures heterogeneity and how variance-based
purification outlined in Figure~\ref{fig:concept} can lead to a large
increase in test accuracy. Results are established for EMNIST image
data and synthetic data. The discussion in
Section~\ref{sec:discussion} includes an entropy perspective on the
heterogeneity captured by the random variable $X$ as well as an
outlook.

\section{Some preliminary facts}
\label{S:basic}

In the following, we will suppress the model class $\mathcal{F}$ thus
identifying $F_\theta$ with $\theta$, and, by common convention, will
consider loss functions of the form
\begin{equation}
\label{eq:loss}
f_Z(\theta) = {1\over n} \sum_{z\in Z} L(z,\theta) \;,
\end{equation}
with $\theta \in \mathbb{R}^m$. The following lemma is well-known; we
include its proof and those of the following lemmata in
Section~\ref{sec:proofs}.
\begin{lemma}
\label{L:1}
Let $f$ be a convex, twice continuously differentiable function
$f\colon \mathbb{R}^n\rightarrow \mathbb{R}$ and
\begin{equation}
\label{eq:f2}
f_{2,\theta_0}(\theta) =
    f(\theta_0) +\nabla f(\theta_0)(\theta-\theta_0)
  + (\theta-\theta_0)^T\nabla^2f(\theta_0)(\theta-\theta_0)
\end{equation}
its second order Taylor approximation at $\theta_0$. Then
$f_{2,\theta_0}$ assumes its minimum at
\begin{equation}
\label{eq:theta_hat}
\hat{\theta} =
\argmin_\theta f_{2,\theta_0}(\theta)
  = \theta_0 - [\nabla^2 f(\theta_0)]^{-1} \nabla f(\theta_0) \;.
\end{equation}
\end{lemma}

Let $Z$ be as in~\eqref{eq:sample}, set $\epsilon = (\epsilon_z)_{z\in
  Z}$, and define the function $u_\epsilon(\theta)$ by
\begin{equation}
  u_\epsilon(\theta)=f_Z(\theta)+\sum_{z\in Z} \epsilon_z L(z,\theta) \;,
\end{equation}
where we by construction have~$u_{0}(\theta) = f_Z(\theta)$. The
function~$u_\epsilon$ can be thought of as an $\epsilon$-deformation
of~$f_Z$ at~$0$.
We set $\hat\theta_{u_\epsilon} = \argmin_\theta
u_{\epsilon}(\theta)$, which depends on $\epsilon$, note that
$\hat\theta_{u_0} = \hat\theta_{f_Z} = \hat\theta$, and write $H =
\nabla^2f_Z(\hat\theta)$ for the Hesse matrix of $f_Z$ at
$\hat\theta$. In the following, we will drop $\epsilon$ from
$u_\epsilon$ and simply write $u$ when it is clear from the context.

\begin{lemma}\label{L:theta-diff}
\begin{equation}\label{E:0}
\frac{\partial}{\partial \epsilon_z} \hat\theta_{u_0}
  = - H^{-1}\nabla L(z,\hat\theta).
\end{equation}
\end{lemma}
In the following, it is convenient to introduce the symmetric,
positive definite, bi-linear form~$\beta_f$ given by
\begin{equation*}
  \beta_f(v,w) = \langle v, w \rangle = v^T H^{-1} w \;.
\end{equation*}

\begin{lemma}\label{L:basic}
\begin{align*}
\frac{\partial}{\partial\epsilon_z}L(z',\hat\theta_{u_{0}})
&= - \nabla L(z',\hat\theta)^T  \left[\nabla^2f(\hat\theta)\right]^{-1}\nabla L(z,\hat\theta) \\
\frac{\partial}{\partial\epsilon_z}L(z',\hat\theta_{u_{0}})
&= - \langle \nabla L(z',\hat\theta),  \nabla L(z,\hat\theta)\rangle
 = \frac{\partial}{\partial\epsilon_{z'}}L(z,\hat\theta_{u_{0}})
\end{align*}
\end{lemma}
The $z/z'$ symmetry in the latter expression has, to our knowledge,
not been previously stated and will be of particular relevance for the
proof of the main result.

\section{Main results}
\label{sec:main}

Let $Z=\{z_1,\dots,z_n\}$ be a sample as in~\eqref{eq:sample} and introduce
$$
X_u(\{z,z'\})= \frac{\partial}{\partial \epsilon_z} L(z',\hat\theta_u)
$$
which is well-defined by Lemma~\ref{L:basic} since
$\frac{\partial}{\partial \epsilon_z}
L(z',\hat\theta)=\frac{\partial}{\partial \epsilon_{z'}}
L(z,\hat\theta)$.
We are interested in studying what happens when a subset $M\subset Z$
is removed from $Z$. To this end, let $\vert M\vert =s \le n = \vert
Z\vert$ and define $\eta_M =(\epsilon_t)_t$ by
$$
\epsilon_t =
\begin{cases}
  -1/n\;, & \text{\rm if } t\in M \\
     0\;, & \text{\rm otherwise,}
\end{cases}
$$
and thus
\begin{equation*}
  u_{\eta_M}(\theta) =
  f_Z(\theta) - \sum_{t\in M} \frac{1}{n} L(t,\theta) \;.
\end{equation*}
Note that the loss function with respect to $Z\setminus M$ is
$f_{Z\setminus M}(\theta) = \frac{1}{n-s} \sum\limits_{t\in Z\setminus
  M}L(t,\theta)$ is different from the function
\begin{equation*}
  u_{\eta_M}(\theta) =\frac{1}{n}\sum_{t\in Z\setminus M} L(t,\theta) \;,
\end{equation*}
derived from $f_Z(\theta)$ by removing all terms $L(t, \theta)$ for
$t\in M$.
However, for any constant $\alpha > 0$ we have $\argmin_\theta f =
\argmin_\theta (\alpha f)$ and thus
\begin{equation*}
    \argmin_\theta f_{Z\setminus M} =\argmin_\theta u_{ \eta_M} \;,
\end{equation*}
that is,~$f_{Z\setminus M}$ assumes a minimum at $\hat\theta_{u_{
    \eta_M}}$.
With $f_{Z\setminus M}$ we have associated loss terms
$L(z',\hat\theta_{u_{\eta_M}})$ which are exhibiting a different
model, specific to the set $M$ of data points being
removed. Accordingly we have the random variable
$$
X_{u_0}=\frac{\partial}{\partial \epsilon_z}L(z',\hat\theta_{u_{0}})
$$
describing the original system with twice continuously differentiable
convex loss function~$f_Z(\hat\theta_{u_0})=\frac{1}{n}\sum_t
L(t,\hat\theta_{u_0})$, and the random variable
$$
X_{u_{\eta_M}}=\frac{\partial}{\partial \epsilon_z}L(z',\hat\theta_{u_{\eta_M}}) \;,
$$
describing the systems where the data set $M$ has been removed with
loss functions~$f_{Z\setminus
  M}(\hat\theta_{u_{\eta_M}})=\frac{1}{n-s}\sum_t L(t,\hat\theta_{u_{
    \eta_M}})$.
We view the moments
$$
\mathbb{E}[X_u^k] =
\frac{1}{\binom{n}{2}} \sum_{\{z,z'\}, z\neq z'\atop z,z'\in Z}
\left[\frac{\partial}{\partial\epsilon_z} L(z',\hat\theta_{u|_{ \eta_M}})\right]^k
$$
as an influence potential of a data set $Z$ for given loss function
$f_Z(\theta)$ having its minimum at~$\hat\theta_{u|_{ \eta_M}}$.

\begin{theorem}\label{T:manyk}
Let $M\subset Z$ with $\vert M\vert=s$, where $s\le n_0$, and $n$ is
sufficiently large.  Furthermore, let $X_{u_0}$, $X_{u_{\eta_M}}$ be
the random variables
$X_{u_0}(\{z,z'\})=\frac{\partial}{\partial\epsilon_z}
L(z',\hat\theta_{u_{ 0}})$ and
$X_{u_{\eta_M}}(\{z,z'\})=\frac{\partial}{\partial\epsilon_z}L(z',\hat\theta_{u_{
    \eta_M}})$, respectively.  Then
\begin{equation}
  \mathbb{E}[X_{u_0}^k]
  - \frac{1}{\binom{n}{s}}\sum_{M\subset  Z\atop \vert M\vert = s}  \mathbb{E}[X_{u_{\eta_M}}^k]=
  \frac{k s}{n-2}  \mathbb{E}[X_{u_0}^k] +O(n^{-2}).
\end{equation}
\end{theorem}
The full proof is given in Section~\ref{sec:proofs} on
page~\pageref{proof:T:manyk}; here we provide an outline of the main
arguments.

We view
$\left[\frac{\partial}{\partial\epsilon_z}L(z',\hat\theta_{u})\right]^k$
as a function of $\epsilon = (\epsilon_t)_t$, form its Taylor
expansion at $(\epsilon_t)_t=0$, and subsequently evaluate at
$\epsilon = \eta_M$.
Next we consider pairs $(M,\{z',z\})$ where $z,z'\not\in M$. Summing
over all pairs $(M,\{z',z\})$ we provide interpretations of the key
expressions
\begin{eqnarray*}
\frac{1}{\binom{n-2}{s}} \frac{1}{\binom{n}{2}}
\sum_{(M,\{z',z\}) \atop \vert M\vert=s} \left[\frac{\partial}{\partial\epsilon_z} L(z',\hat\theta_{u_{  0}}) \right]^k  & = & \mathbb{E}[X_{u_0}^k]\\
\frac{1}{\binom{n-s}{2}} \sum_{\{z,z'\}\cap M=\varnothing}  \left[\frac{\partial}{\partial\epsilon_z}
    L(z',\hat\theta_{u_{ \eta_M}})\right]^k & = &  \mathbb{E}[X_{u_{\eta_M}}^k],
\end{eqnarray*}
as well as
\begin{eqnarray*}
\frac{1}{\binom{n}{s}} \frac{1}{\binom{n-s}{2}}  \sum_{(M,\{z',z\}) \atop \vert M\vert =s}  \left[\frac{\partial}{\partial\epsilon_z}
  L(z',\hat\theta_{u_{ \eta_M}})\right]^k
& = & \frac{1}{\binom{n}{s}}\sum_{M\subset  Z\atop \vert M\vert =s}  \mathbb{E}[X_{u_{\eta_M}}^k].
\end{eqnarray*}
Since $f_{Z\setminus M}$ assumes a minimum at
$\hat\theta_{u_{\eta_M}}$ we derive from these interpretations
\begin{align*}
  \Delta_{s}(\mathbb{E}[X_{u_0}^k])
  &=
  \rho_s \sum_{(M,\{z',z\})} \left[\frac{\partial}{\partial\epsilon_z} L(z',\hat\theta_{u_{0}}) \right]^k -
  \left[\frac{\partial}{\partial\epsilon_z}L(z',\hat\theta_{u_{\eta_M}})\right]^k \\
&= k \rho_s \sum_{\{z,z'\}} \left[\left[ \frac{\partial}{\partial\epsilon_z}
    L(z',\hat\theta_{u_{0}}) \right]^{k-1} \left[\sum_{M} \sum_{y\in M}\left[\frac{\partial^2}
       {\partial \epsilon_y\partial \epsilon_z} L(z',\hat\theta_{u_{0}} ) \cdot
    \left(\frac{1}{n}\right)  + O(n^{-2})\right]\right ]\right]  .
\end{align*}
We then observe that the inner double sum on the RHS simplifies to
\begin{equation}\label{E:reorganize0}
\binom{n-3}{s-1}
\sum_{y\in Z\atop y\neq z,z'}
\left[
  \frac{\partial^2} {\partial \epsilon_y\partial \epsilon_z} L(z',\hat\theta_{u_{0}} )
  \left(\frac{1}{n}\right)+O(n^{-2})
\right]
\end{equation}
which allows us employ Lemma~\ref{L:claim1} and subsequently derive
\begin{align*}
   \Delta_s(\mathbb{E}[X_{u_0}^k] &=
  \mathbb{E}[X_{u_0}^k] - \frac{1}{\binom{n}{s}} \sum_{M\subset  Z\atop \vert M\vert =s}
  \mathbb{E}[X_{u_{\eta_M}}^k] \\
&=   \frac{k s}{n-2}  \mathbb{E}[X_{u_0}^k]   +   \frac{k s}{n-2}  O(n^{-1+2-2})\;.
\end{align*}

It is worth noting that in case of the removal of large sets $M$, say
$\vert M\vert =O(n)$, the $M$-indexed Taylor expansion, as detailed in
Equation~\eqref{E:local44} of Section~\ref{sec:proofs}, does not
produce information about individual, $M$-indexed terms
$\left[\frac{\partial}{\partial\epsilon_z} L(z',\hat\theta_{u_{
      \eta_M}})\right]^k$ since the total differential involves $O(n)$
terms and the quadratic term error term in Equation~\eqref{E:local44}
becomes a constant.

Theorem~\ref{T:manyk} is concerned with the average over all such
$M$-indexed terms and as the proof shows these averages,
$\frac{1}{\binom{n}{s}} \sum_{M\subset Z\atop \vert M\vert =s}
\mathbb{E}[X_{u_{\eta_M}}^k]$, are expressed via
Equation~\eqref{E:reorganize}, see Section~\ref{sec:proofs} for
details, as multiples of a sum over all $y\in Z$, $y\neq z,z'$. The
condition $s\le n_0$ and $n$ being sufficiently large assures that the
Taylor expansion approximation of
$\left[\frac{\partial}{\partial\epsilon_z}L(z',\hat\theta_{u_{\eta_M}})\right]^k$
and in particular the quadratic error term in
Equation~\eqref{E:local44} are well-defined.

The following statements represent basic facts needed to prove Lemma
\ref{L:claim0}, which in turn implies Lemma~\ref{L:claim1}.
\begin{lemma}\label{L:double-diff-theta}
\begin{equation*}
  \frac{\partial^2}{\partial \epsilon_z\partial \epsilon_x} \hat\theta_{u_{0}}=
H^{-1} \left[\nabla^2 L(z,\hat\theta)\right] H^{-1}  \cdot \nabla L(x,\hat\theta)
+ H^{-1} \left[\nabla^2 L(x,\hat\theta)\right] H^{-1}  \cdot \nabla L(z,\hat\theta).
\end{equation*}
\end{lemma}
(Proof on page~\pageref{proof:L:double-diff-theta}.)

\begin{lemma}\label{L:double-diff}
\begin{eqnarray*}
\frac{\partial^2}{\partial \epsilon_x \partial\epsilon_z} L(z',\hat\theta_{u_{0}})
&=&
-\big\langle  \frac{\partial}{\partial \epsilon_z}  \nabla L(z',\hat\theta_{u_0}) , \nabla L(x,\hat\theta_{u_0})\big\rangle + \\
    & &  \langle \nabla L(z',\hat\theta_{u_{0}}),  \nabla^2 L(x,\hat\theta_{u_{0}})H_{u_{0}}^{-1}\nabla L(z,\hat\theta_{u_{0}})\rangle  +\\
    & & \langle \nabla L(z',\hat\theta_{u_{0}}), \nabla^2 L(z,\hat\theta_{u_{0}}) H_{u_{0}}^{-1} \nabla L(x,\hat\theta_{u_{0}})\rangle
\end{eqnarray*}
\end{lemma}
(Proof on page~\pageref{proof:L:double-diff}.)
\begin{lemma}\label{L:claim0}
\begin{equation}\label{E:noerror}
  \frac{1}{n}\sum_{y\in Z}    \frac{\partial^2} {\partial \epsilon_y\partial \epsilon_z} L(z',\hat\theta_{u_{0}} )=
   \frac{\partial}{\partial\epsilon_z} L(z',\hat\theta_{u_{0}}) \;.
\end{equation}
\end{lemma}
(Proof on page~\pageref{proof:L:claim0}.)
\begin{lemma}\label{L:claim1}
  For any fixed pair $z'\neq z$, $z,z'\in Z$ holds
\begin{equation}\label{E:error}
  \sum_{y\in Z, y\neq z,z'}   \left[ \frac{\partial^2} {\partial \epsilon_y\partial \epsilon_z} L(z',\hat\theta_{u_{0}} )
    \left(\frac{1}{n}\right)+O(n^{-2})\right]=
   \frac{\partial}{\partial\epsilon_z} L(z',\hat\theta_{u_{0}}) +O(n^{-1}) \;.
\end{equation}
\end{lemma}
(Proof on page~\pageref{proof:L:claim1}.)
With the help of Theorem~\ref{T:manyk} we are now in position to
formulate our main result:
\begin{theorem}\label{T:variance}
Let $M\subset Z$ with $\vert M\vert=s$, where $s\le n_0$ and $n$ is
sufficiently large.  Let $X_{u_0}$ and $X_{u_{\eta_M}}$ be the random
variables expressing the influence between data-points $z,z'$ in $Z$
and $Z\setminus M$, respectively. Then we have
\begin{equation}
 \mathbb{V}[X_{u_0}]  - \frac{1}{\binom{n}{s}} \sum_{M\subset  Z\atop \vert M\vert = s} \mathbb{V}[X_{u_{\eta_M}}]  =
  \frac{2 s}{n-2}  \mathbb{V}[X_{u_0}] +O(n^{-2}) \;.
\end{equation}
\end{theorem}
(Proof on page~\pageref{proof:T:variance}.)

The following corollary of Theorem~\ref{T:variance} is the cornerstone
of this paper. It establishes that a descent on the variance
$\mathbb{V}[X]$ is always possible, subject to certain data size
conditions.
\begin{corollary}\label{C:1}
Suppose we are given a data set $Z$ and a twice differentiable, convex
loss function $f_Z(\theta)=\frac{1}{n}\sum_z L(z,\theta)$ with a
minimum at $\hat\theta$ and potential $0<\mathbb{E}[X^k_{u_0}] $. Then
for $s\le n_0$ and sufficiently large $n$, there exists a subset
$M_0\subset Z$ with $\vert M_0\vert=s$ such that
\begin{equation}\label{E:exists}
    \frac{ 2 s}{n-2}  \mathbb{E}[X^k_{u_0}] + O(n^{-2})
    \le
    \mathbb{E}[X^k_{u_0}]   -\mathbb{E}[X^k_{u_{\eta_M}}].
\end{equation}
Furthermore
\begin{equation}\label{E:exists2}
    \frac{ 2 s}{n-2}  \mathbb{V}[X_{u_0}] + O(n^{-2})
    \le
    \mathbb{V}[X_{u_0}]   -\mathbb{V}[X_{u_{\eta_M}}].
\end{equation}
\end{corollary}
(Proof on page~\pageref{proof:C:1}.)

\section{Applications}\label{S:appl}

In this section we present two proof-of-concept case studies that put
the following two conjectures to the test:\\
(a) variance encapsulates data heterogeneity and inversely correlated to test accuracy, and\\
(b) variance-based removal of sample points causes a drop in variance and an increase in test accuracy.

\begin{figure}[ht]
\centerline{\includegraphics[width=0.5\linewidth]{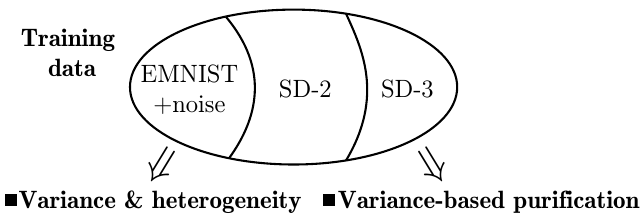}}
\caption{Outline: three classes of data each analyzed with respect to
  variance, heterogeneity, test accuracy and subjected to
  variance-based purification. SD-2 and SD-3 denote synthetic data
  containing two and three distinct distributions, respectively.}
\label{fig:experiment_overview}
\end{figure}

As depicted in Figure~\ref{fig:experiment_overview}, our experiments
involve three classes of data: a subset of the EMNIST
data~\cite{Cohen:17} and two varieties of a synthetic data set. For
the EMNIST data, samples of given sizes containing two distribution
are constructed through mislabeling a certain fraction~$r$ of the
training data. (The two distributions being represented by the
correctly labeled images and the incorrectly labeled images.) For the
synthetic data, denoted by SD-$K$ and described below, we synthesize
training data of prescribed sizes containing a specified number $K$ of
distributions at prescribed ratios.
Each data set is then probed using the following experiments: (1)
tracking of the variance $\mathbb{V}[X]$ and test accuracy as a
function of the composition ratio of the distributions, and (2)
measurements of the variance~$\mathbb{V}[X]$ and testing accuracy
under variance-based sample point removals.
We note that while high accuracy can always be achieved on the
training set by increasing model complexity or extending the number of
training epochs, high accuracy on the independent test set reflects
the model’s true predictive capability. Test accuracy thus serves as a
reliable benchmark for assessing the predictive power and
generalization ability of a trained model.
In each case, multinomial logistic regression (MLR) was employed using
the PyTorch~\cite{pytorch:19} library and its ADAM optimizer with
parameters as detailed below. We remark that these are preliminary
experiments meant to illustrate our framework, and refer to
Section~\ref{sec:discussion} for a discussion of implications, use of
deep learning architectures, and design of engineering principles and
algorithms targeting scaling and generalization.

\textbf{Data description.} For the EMNIST data~\cite{Cohen:17} we used
all~10 digits, each represented by~$10^3$ data points for a total
sample size of~$|Z|=10^4$ distributed evenly for each digit.
The synthetic data sets are constructed using a family of functions
defined in the following manner. Each function is of the form
$\phi^\ell \colon E^{10} \longrightarrow E$ for the alphabet $E = \{R,
D, B, N\}$, and is defined in terms of positional real-valued weights
$q^\ell = (q^\ell_i)_{i=1}^{10}$ (i.e., feature weights), alphabet
weights $w^\ell = (w^\ell_i)_{i\in E}$, and thresholds $a^\ell =
(a^\ell_i)_{i=1,2,3}$ by
\begin{equation}
\phi^\ell(x) =
\begin{cases}
R, & \text{if }   \phantom{a_1<}   C(x) \le a_1\;, \\
D, & \text{if } a_1 < C(x) \le a_2\;, \\
B, & \text{if } a_2 < C(x) \le a_3 \;, \\
N, & \text{if } a_3 < C(x) \;.
\end{cases}
\end{equation}
where $C(x)=\sum_{i=1}^{10}q_i^\ell w^\ell(x_i)$.  A sample $Z$ of
size $N = N_1 + N_2 + \cdots N_K$ is constructed as the union of sets
whose $N_i$ elements are pairs $(x,\phi^i(x))$ where the feature
vectors $x$ are sampled from $U(E^{10})$.
The specific distribution parameters are given as follows:
\begin{verbatim}
  label_weights_1: R:  1.0, D: -1.0, N:  0.3, B: -0.3
  label_weights_2: R: -0.4, D:  1.0, N: -0.8, B:  0.6
  label_weights_3: R:  1.2, D:  0.4, N: -0.2, B: -0.8

  feature_weights_1: [  1,  -1,  1, -1,   1,   -1,  1,   -1,    1,   -1]
  feature_weights_2: [0.5, 0.5, -1,  1,   0, -0.5,  1,   -1,  0.5, -0.5]
  feature_weights_3: [  1, 0.5,  1,  1, 0.5, -0.5, -1, -1.5, -0.5,    1]

  threshold_1: a_3=2.0, a_2=0.0, a_1=-2.0
  threshold_2: a_3=1.5, a_2=0.0, a_1=-1.5
  threshold_3: a_3=2.0, a_2=0.0, a_1=-2.5
\end{verbatim}

\subsection{EMNIST image data}
\label{sec:image_data}

We provide next our findings about relating variance with data
heterogeneity and test accuracy as well as the effect of
variance-based purification. We display variance and predication
accuracy relation as a combination of the mixing ratio of correctly
and incorrectly labeled EMNIST image data, and the effect of variance
based removal of certain training points.
In Figure~\ref{fig:emnist_potential_accuracy_errorrate}, we showcase
variance~$\mathbb{V}[X]$ and accuracy as functions of the error
rate~$r$.
\begin{figure}[ht!]
\centerline{\includegraphics[width=0.75\linewidth]{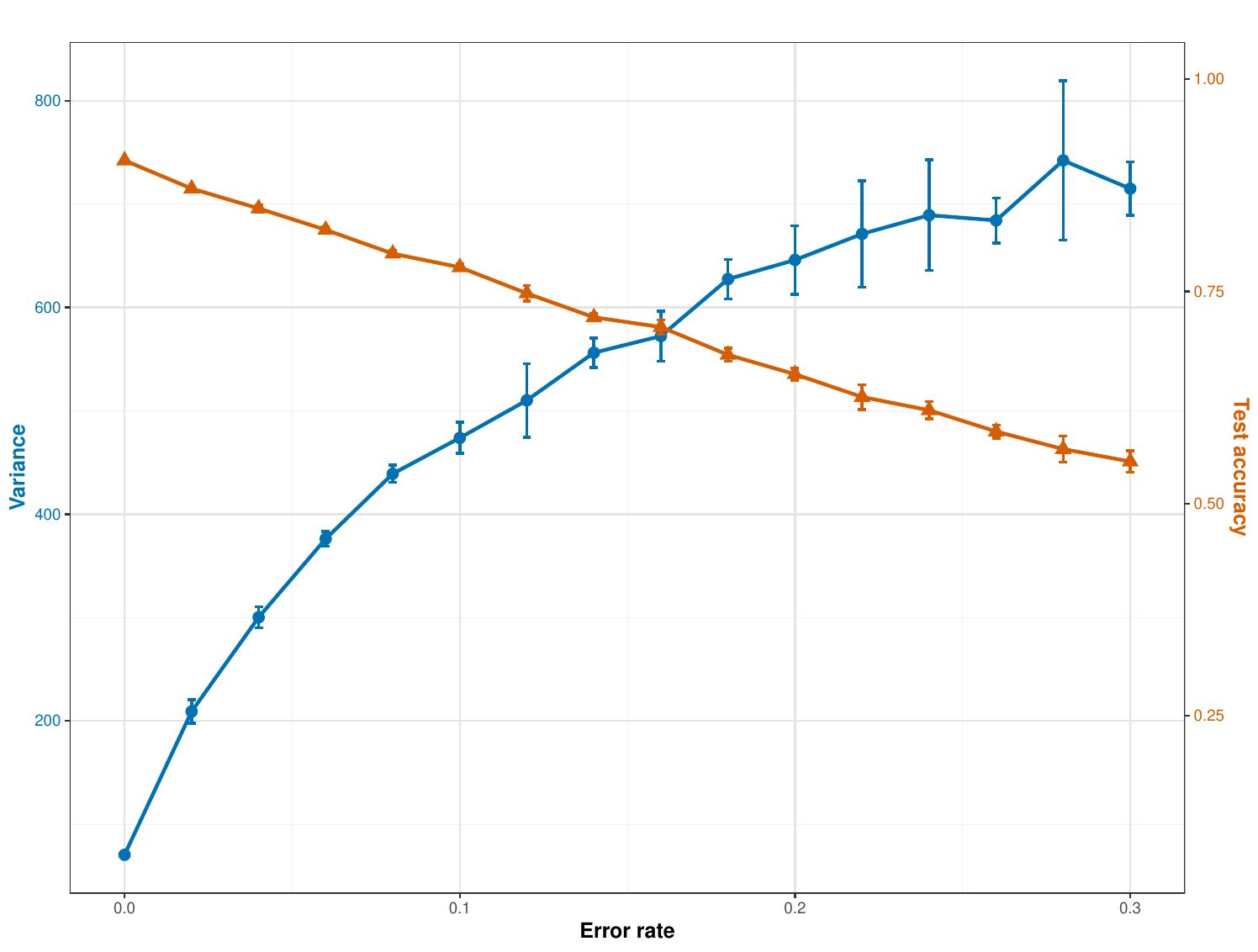}}
\caption{Variance $\mathbb{V}[X]$ (blue curve) and prediction accuracy
  (orange curve) for EMNIST data as a function of the error rate~$r$
  applied to training data. All~10 digits were used, with each
  digit~$0$ through~$9$ represented by~$10^3$ images to obtain~$|Z|=
  10^4$. We train an MLP model which takes all $28 \times 28$ gray
  scale pixels as features, with learning rate~$10^{-3}$ and batch
  size $64$, and use the trained model to predict on a disjoint test
  set of size $10^4$ sampled from EMNIST data.  The horizontal axis
  gives the error rate, and the left (resp. right) vertical axis gives
  the variance~$\mathbb{V}[X]$ (resp. test accuracy). Filled
  circles/triangles show expectations, while bars show standard
  deviations across the~5 replications generated for each error
  rate~$r$.
Note that the error rates for test accuracy are small and that nearly
all error bars are contained inside the triangle showing the
expectation.  }
\label{fig:emnist_potential_accuracy_errorrate}
\end{figure}

We proceed by illustrating how the variance $\mathbb{V}[X]$ can be
used to improve prediction accuracy through elimination of sample
points under MLR. For the case shown in
Figure~\ref{fig:mnist_potential_accuracy_misclassification_purification}
we again used EMNIST data, but this time with $|Z|=1200$ images
composed of $n_4=600$ digits $'4'$ and $n_8 = 600$ digits $'8'$. The
sample $Z$ was split into a training set $A_1$ with $|A_1| = 600$,
with each digit evenly distributed, and a test set $A_2$ with $|A_2| =
1200 - |A_1| = 600$, and for the training set $A_1$ we randomly
mislabeled a fraction~$r=0.30$ of the points.  Finally, after initial
training we measure the initial variance~$\mathbb{V}[X]$ and testing
accuracy, corresponding to the values for iteration~0 in
Figure~\ref{fig:mnist_potential_accuracy_misclassification_purification}.

Next we iterated the following procedure: for each $z \in A_1 = A_1^0$
train the MLR model on~$A^0_{1,-z} = A_1^0 \setminus {z}$, aka
leave-one-out (LOO), and record test accuracy. Following this, measure
\begin{equation*}
  \delta(z, A^0_1) = \mathbb{V}[A^0_1] - \mathbb{V}[A^0_{1,-z}] \;,
\end{equation*}
and then remove the $k=10$ points of $A^0_1$ with the largest value
of~$\delta$ to obtain $A_1^1$, concluding iteration~1. The MLR model
is then retrained on $A_1^1$, $k=10$ points are removed
mapping~$A_1^1$ into~$A_1^2$, and so on.

Figure~\ref{fig:mnist_potential_accuracy_misclassification_purification},
shows that variance drops quickly until about iteration~4 while the
test accuracy steadily increases for about~$20$ iterations which
corresponds to removal of~$200$ sample points from~$A_1^0$ (shown by
vertical dashed line). This ``purification'' procedure identifies a
subset $A_1' \subset A_1$ that allows for a significantly higher test
accuracy despite~$|A_1'| < |A_1|$.
\begin{figure}[ht!]
\centerline{\includegraphics[width=0.75\linewidth]{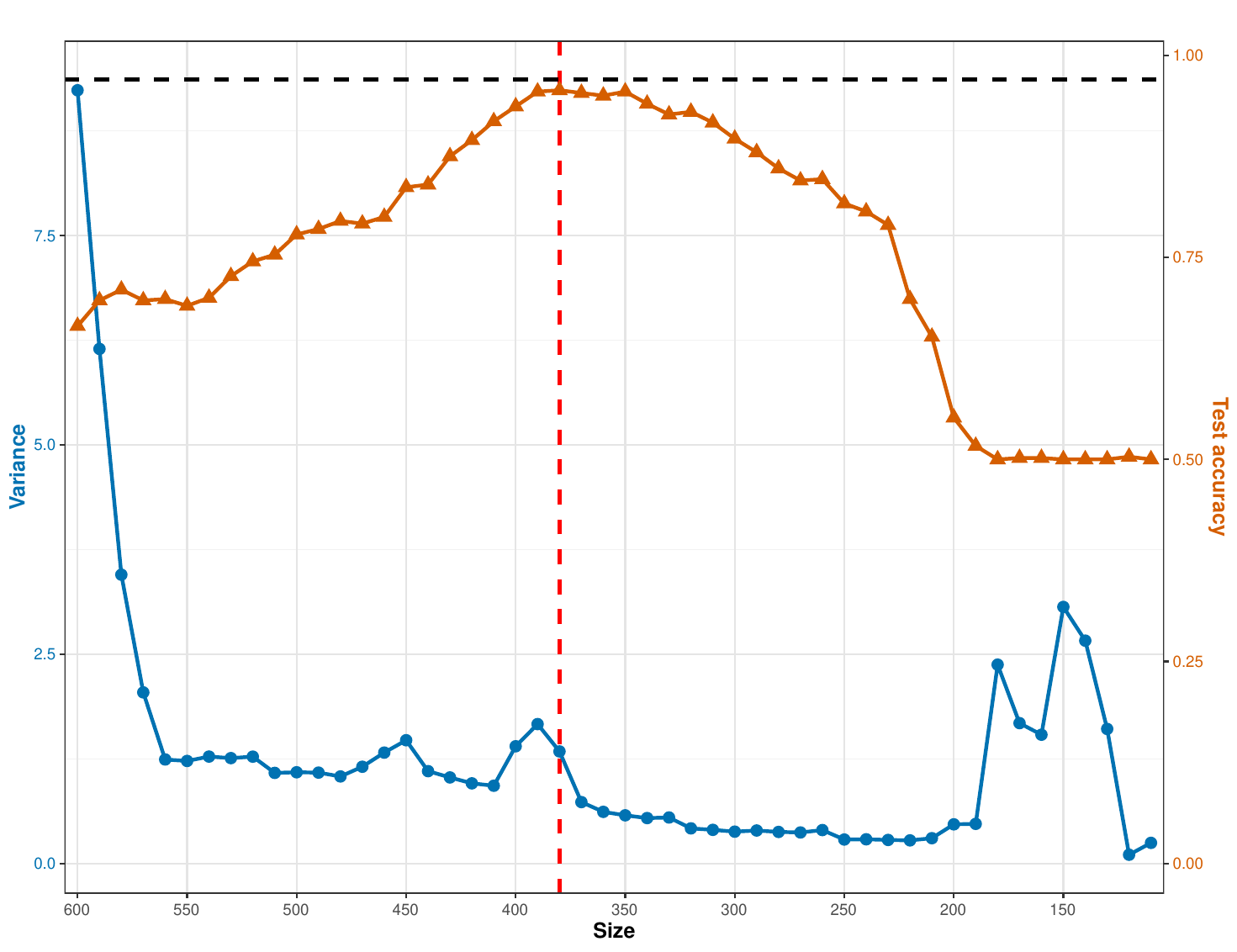}}
\caption{The evolution of training set variance $\mathbb{V}[X]$ (blue
  curve) and test accuracy (orange curve) over~20 purification
  iterations. Here a training set $T$ consists of 600 EMNIST samples
  digits '4' and '8' (evenly distributed) with an error
  rate~$r=0.30$. An independent test set~$S$ contains~600 correctly
  labeled EMNIST samples with the same digit distributions
  as~$T$. Training was conducted using MRL with batch size~64 and
  learning rate~$10^{-3}$. For reference, we trained a model on a
  single distribution training set $T_0$ of size $600$, evaluated on
  $S$, obtaining a test accuracy of $0.97$ at $r=0$, shown as the
  dashed black line. The maximal accuracy reached during purification
  was~$0.957$ which occurred when~$220$ data points had been removed,
  shown at the red dashed line.}
\label{fig:mnist_potential_accuracy_misclassification_purification}
\end{figure}

\subsection{Synthetic data - two distributions}
\label{sec:voter2_data}

Using synthetic data induced by two distributions and a total sample
size~$|Z|=600$ and MLR, Figure~\ref{fig:mix_potential_accuracy}
indicates a clear pattern: data heterogeneity negatively impacts
prediction accuracy, and the variance, $\mathbb{V}[X]$, correlates
positively with heterogeneity. Here accuracy reaches a minimum near
the~50/50 mix, at which point the potential reaches its maximal value.
\begin{figure}[ht!]
\centerline{\includegraphics[width=0.75\linewidth]{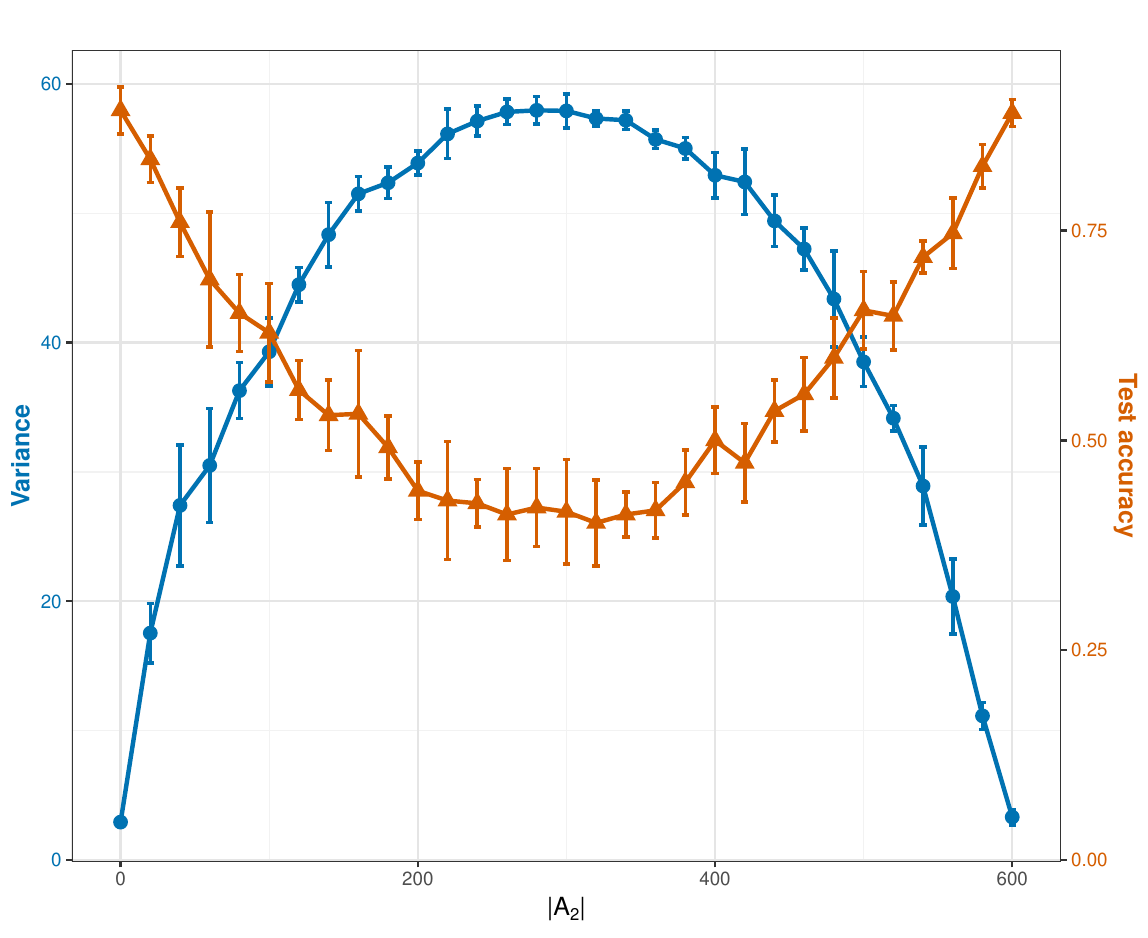}}
\caption{Variance $\mathbb{V}[X]$ and test accuracy as functions of
  mixed synthetic data (SD-2) with two distributions ($A_1$, $A_2$)
  and~$|Z| = 600$ using MLR. The horizontal axis shows $|A_2|$ with
  $|A_1|=|Z| - |A_2|$.  The left vertical axis shows test variance
  (blue curve) while the right axis shows the test accuracy (orange
  curve). $Z$ is randomly split into training and test sets in an 8:2
  ratio. The test accuracy is computed by training the model on the
  training set and evaluating it on the test set.  Filled
  circles/triangles give the expectations while error bars show
  standard deviations; for each composition~$k=5$ replications were
  generated. For the MLR training batch size was~$64$ and learning
  rate was~$10^{-3}$.}
\label{fig:mix_potential_accuracy}
\end{figure}

Figure~\ref{fig:voter_2_purification} represents the analogue of
Figure~\ref{fig:mnist_potential_accuracy_misclassification_purification}
which shows the variance-based, iterated purification using
Leave-One-Out under MLR as detailed in the previous section for a
sample~$Z$ consisting of $10^3$ points. Note that at iteration~100,
all data points would have been removed. One may speculate regarding
the relation between the maximal test accuracy and the inflection
point in the variance.
\begin{figure}[ht!]
\centerline{\includegraphics[width=0.75\linewidth]{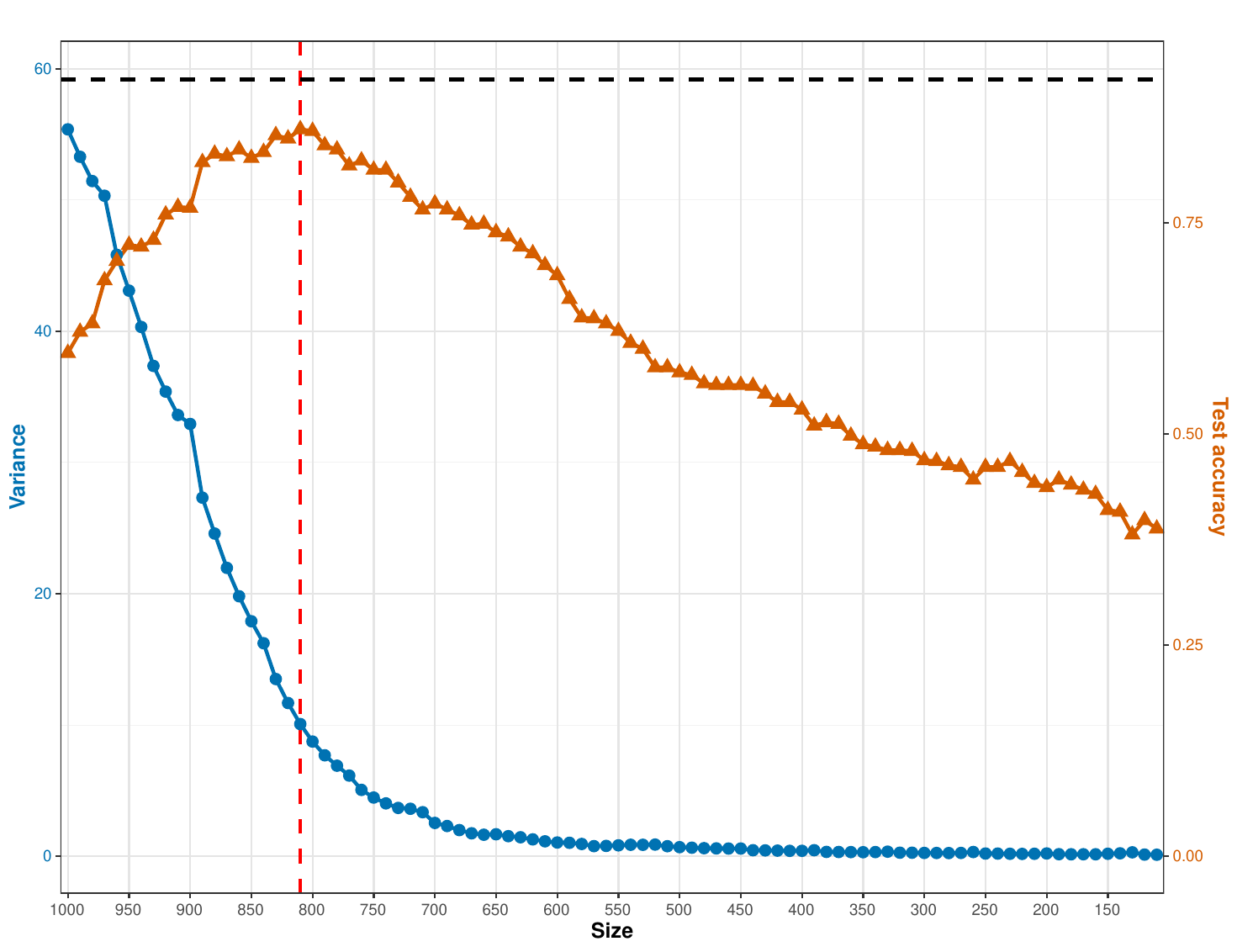}}
\caption{Evolution of training set variance $\mathbb{V}[X]$ (blue
  curve) and test accuracy (orange curve) over 20 purification steps
  using synthetic data SD-2 and MRL. Here the SD-2 training set $Z$ is
  composed of distributions $A_1$ and $A_2$ with $|A_1| = 700$ and
  $|A_2| = 300$; the test set $S$ consists of $1000$ data points
  following the distribution $A_1$.
Training was conducted using MRL with batch size~64 and learning
rate~$0.001$. For reference, we trained a model on a set $Z_0$
consisting of $10^3$ data points following distribution $A_1$ and
evaluated on~$S$ obtaining a test accuracy is~0.94 (black dashed
line). During purification, the maximal test accuracy was measured to
be~$0.86$ occurring after removal of 190 data points (red dashed
line).}
\label{fig:voter_2_purification}
\end{figure}

\subsection{Synthetic data - three distributions}

It is natural to ask how the observations from the previous two
sections translate to a sample where three or more distributions are
present. Referring to Figure~\ref{fig:synthetic_1000_3_mix_variance}
we illustrate the variance~$\mathbb{V}[X]$ under MLR for a
3-distribution synthetic data where the sample $Z$ is fixed at size
$600$, and where the three distributions are denoted by $A_1$, $A_2$,
and $A_3$.  The projections onto the planes $|A_2|=0$ and $|A_3|=0$
clearly match the structure of the variance shown in
Figure~\ref{fig:mix_potential_accuracy}. It is also appears that
variance captures the degree of heterogeneity being larger near the
``diagonal'' where $|A_1| = |A_2| = |A_3|$.

Similarly, in Figure~\ref{fig:voter_data_700_150_150_purification} we
show variance-based LOO purification for a~700/150/150
mixed~3-distribution synthetic data sample $Z$ using MLR. Again we
observe a steady increase in test accuracy and a steady drop in
variance $\mathbb{V}[X]$ for about~$20$ iterations. Here test accuracy
ranges from about~$0.65$ for the initial sample~$Z$ to almost~$0.85$
upon targeted removal of about 20\% of the sample points. One also
notices that test accuracy peaks near the inflection point of the
variance.

\begin{figure}[ht!]
\framebox{\includegraphics[height=136bp]{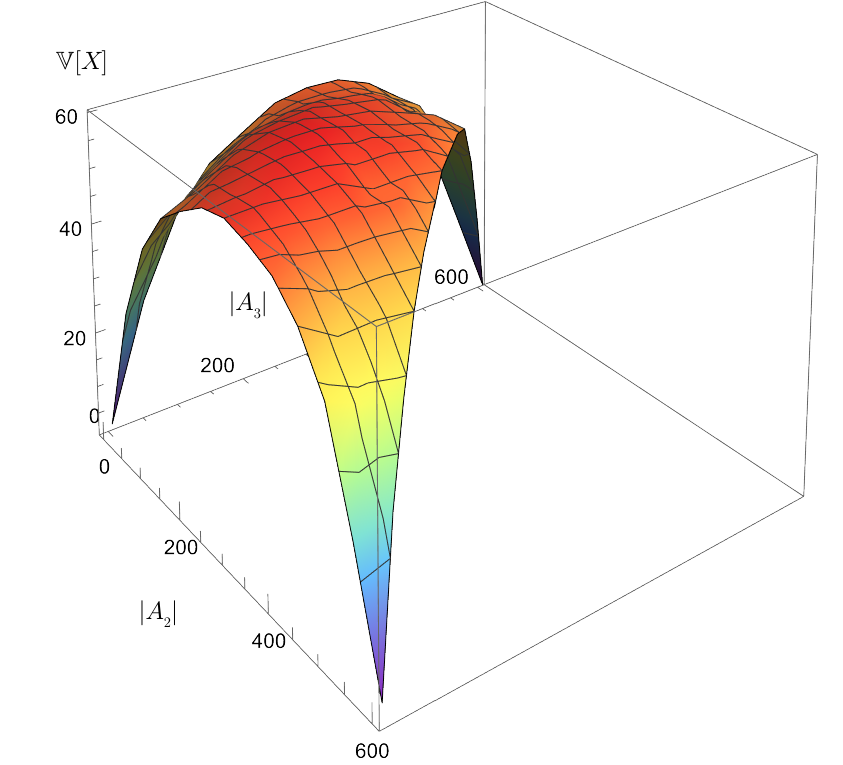}}\quad
\framebox{\includegraphics[height=136bp]{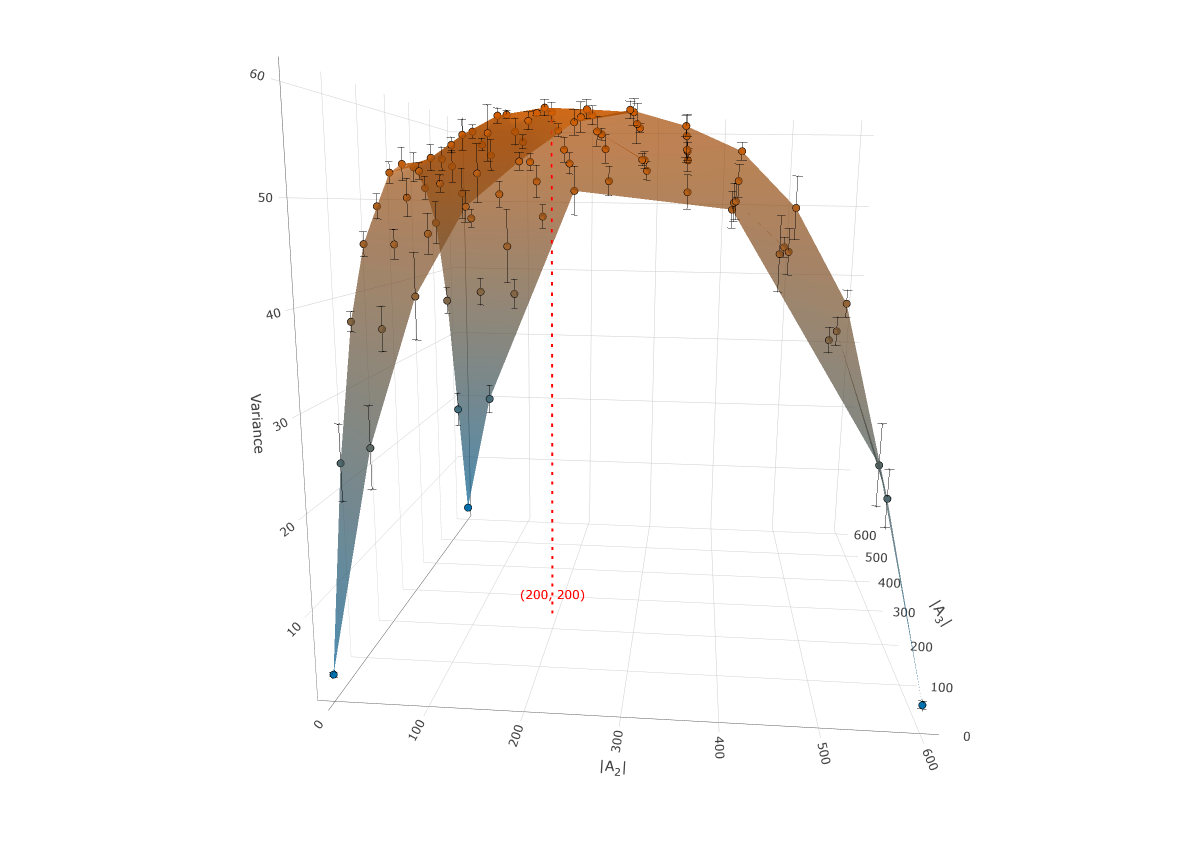}}
\caption{Variance $\mathbb{V}[X]$ as a function of sample composition
  for SD-3 where $|Z| = |A_1| + |A_2| + |A_3| = 600$. Variance
  ($z$-axis) is shown as a function of $(|A_2|, |A_3|)$ where $|A_2|$
  and $|A_3|$ are shown on the $x$- and $y$-axis, respectively. Here
  MLR was used with the synthetic SD-3 data. On the right, filled
  circles represent expectations and error bars encode standard
  deviations. Each data point was generated using $k=5$
  replications. For the MLR training the batch size was~$64$ and the
  learning rate was~$0.001$. The variance peaks at $|A_1| = |A_2|=
  |A_3|=200$ as indicated by the red dashed line on the right.  }
\label{fig:synthetic_1000_3_mix_variance}
\end{figure}

\begin{figure}[ht!]
\centerline{\includegraphics[width=0.75\linewidth]{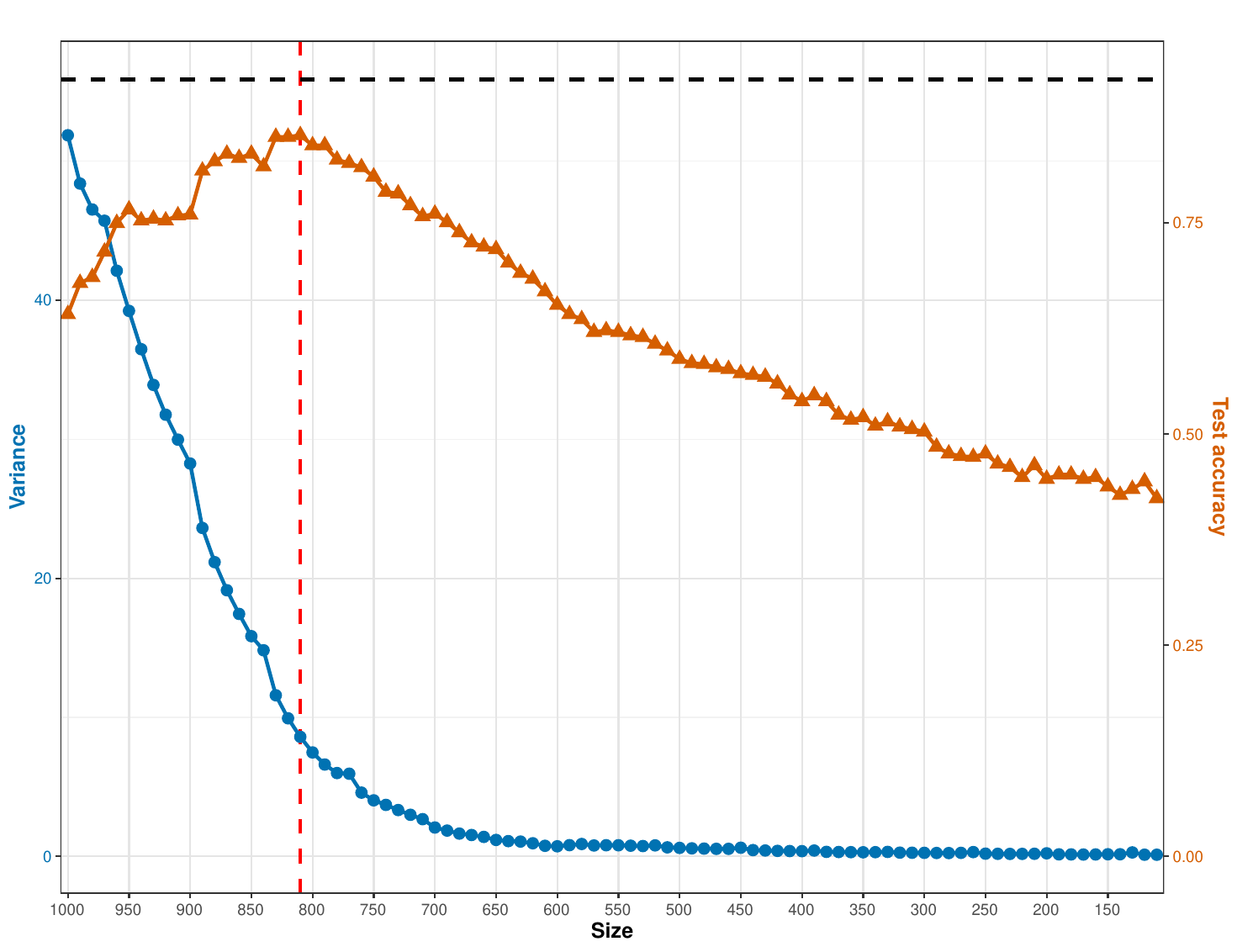}}
\caption{The evolution of training set variance $\mathbb{V}[X]$ (blue
  curve) and test accuracy (orange curve) over 20 purification steps
  using synthetic SD-3 data. Here~$Z$ is composed of distributions
  $A_1$, $A_2$ and $A_3$ with $|A_1| = 700$, $|A_2| = 150$ and
  $|A_3|=150$. The test set consists of~$10^3$ data points following
  the distribution~$A_1$. Training was conducted using MRL with batch
  size~$64$ and learning rate~$10^{-3}$. For reference, we trained a
  model $Z_0$ consisting of~$10^3$ data points following the
  distribution $A_1$ and evaluated it on~$S$, obtaining
  accuracy~$0.94$ as shown by the black dashed line. The maximal test
  accuracy observed during purification was $0.86$ which occurred
  after removal of~190 data points (red dashed line).}
\label{fig:voter_data_700_150_150_purification}
\end{figure}

\section{Discussion}
\label{sec:discussion}
In this paper we have introduced a system-level random variable $X$
representing the influences among all training data points. Its basic
properties allow characterization and manipulation of ensemble
properties, supporting the goal of improving prediction accuracy.
This perspective has immediate implications: random variables are
understood via their moments which become central to our framework.
The question then becomes: can we 
insights on moments to deeper insights into the structure of training
data?
Assuming convexity, Theorems~\ref{T:manyk} and~\ref{T:variance} show
that the second centered moment $\mathbb{V}[X]$ can support a form of
data purification in which sub-samples of data are produced that
reduce variance.  Theorem~\ref{T:variance} and its corollary are
tantamount to an existence proof of a purification algorithm,
immediately raising two questions: (a) how does variance relate to
data heterogeneity, and (b) can we construct efficient variance-based
algorithms that produce subsets of data on which prediction accuracy
is improved?

To ground our statements and results, we need to make precise what we
mean by data heterogeneity.  Here heterogeneity is relative to the
context within which the set data $Z$ is analyzed: it depends on the
choice of machine learning architecture and loss function, with a loss
function implying that training points are of the form~$z=(x,y)$.
This context determines what we can infer from the data, and
determines which distributions we can recognize.

Suppose we fix the context and that~$\ell$ different distributions are
present in the data where distributions are represented by
functions~$f_i\colon X \rightarrow Y$.
heterogeneity is tantamount to the entropy
$$
E(Z)=-\sum_{i=1}^\ell p_i\log(p_i)\;,
$$
where~$p_i$ denotes the fraction of data consistent with
function~$f_i$.  Clearly, if only one function is present we have
$E(Z)=0$ which is minimal, and one may call the system pure. On the
other hand,~$E(Z)$ is maximal with value $\log (\ell)$ if for given
$\ell \ge 2$ the data points are uniformly distributed among the
$\ell$ blocks.  Random data represents the extreme case where each
data point represents its own distribution and heterogeneity is
maximized at~$\log(|Z|)$.

Given a context and an unknown training set $Z$ we cannot directly
infer its heterogeneity, which is why we introduced the variance which
can be used as a proxy for this.  By design, we knew the composition
of the data in all our experiments in Section~\ref{S:appl}, which
allowed us to showcase the relation between variance and
heterogeneity, demonstrated in
Figures~\ref{fig:emnist_potential_accuracy_errorrate},~\ref{fig:mix_potential_accuracy},
and~\ref{fig:synthetic_1000_3_mix_variance} where~$\mathbb{V}[X]$
faithfully captures data
heterogeneity. Figure~\ref{fig:mnist_potential_accuracy_misclassification_purification}
shows that removing ``noise'' from the image data leads to a drop
in~$\mathbb{V}[X]$.  Obviously, this also decreases entropy by
increasing the fraction of correctly labeled image data.  In the case
of three distributions, Figures~\ref{fig:mix_potential_accuracy}
and~\ref{fig:synthetic_1000_3_mix_variance} show that~$\mathbb{V}[X]$
assumes its maximum when all functions have equal support within the
data, matching the behavior of entropy.

It is not within the scope of this work to present a polished,
computational efficient framework for variance-based optimization of
prediction accuracy.  However, our proof-of-concept examples have
demonstrated that variance-based purification via Corollary~\ref{C:1}
can lead to significant improvements of test accuracy. This suggests
that variance, which is a quantity that can be determined and/or
estimated from the training data, supports a principled identification
of ``outliers''.
As can bee seen in all the experiments, variance-based purification
goes through a regime of increased test accuracy, but eventually
transitions into a second regime where additional removal has a
negative impact. Here it is important to note that as more and more
data points are removed (the case of large $|M|$), we no longer
have~$O(n^{-2})$ error terms as stipulated in
Theorem~\ref{T:variance}. From this point onward, finite size effects
render variance less and less informative.

The experiments presented here were based on leave-one-out (LOO) with
retraining, viable mainly for simpler models and conservatively sized
training data. In ongoing work, we are exploring the use of influence
(see Lemma~\ref{L:basic} and~\cite{Koh-Liang:2018,Cook-Weisberg:1980})
as a means to identify the sets $M$ in a computationally cheaper
manner. Initial results indicate that this approach works well,
including for architectures such as deep neural networks where
convexity assumptions fail~\cite{LeCun:15}.
Preliminary experiments using DL architectures and variance-based
construction of the sets~$M$ (not in this paper) based on retraining
show patterns quite similar to those observed in
Figures~\ref{fig:voter_2_purification}
and~\ref{fig:voter_data_700_150_150_purification}.  Since variance as
a concept does not depend on convexity, it appears to represent a tool
to purify data.  \cite{basu2020influence, epifano2023revisiting,
  hammoudeh2024training, li2024influence} show that an analytical
approach based on convexity is not applicable. The development of a
framework suitable for DL-architectures is work in progress.

In ongoing work we are exploring the engineering of efficient
variance-based purification algorithms and stopping criteria for the
removal process. Figure~\ref{fig:concept} shows our vision for how
this framework could advance supervised learning. Sample data is
subjected to initial training where the purpose is purification rather
than prediction. Following this, the data is partitioned into pure
blocks followed by separate, local training on each pure block to
obtain the sub-models using conventional architectures.  Since the
blocks are more homogeneous, one may be able to use less complex
architectures and still obtain increased accuracy. In this framework,
prediction follows the use of a classifier that routes new data to the
appropriate sub-model(s).

Finally we are seeking to formally understand how~$\mathbb{V}[X]$ and
entropy relate which would bridge this work and the
information-theoretic framework developed in~\cite{Bura:24} for RNA
sequence clustering. The fact that $\mathbb{V}[X]$ mimics the behavior
of entropy is likely not a coincidence. Analysis of higher order
moments such as kurtosis is likely to yield additional insight that
can support data analysis.

\section{Proofs}
\label{sec:proofs}

\subsection{Proofs of basic facts}

\begin{proof}[Proof of Lemma~\ref{L:1}]
\label{proof:L:1}
By convexity of $f$ it follows that $f_{2,\theta_0}$ is convex and
thus its minimum is attained where $\nabla f_{2,\theta_0}(\theta) =
0$. From~\eqref{eq:f2} on page~\pageref{L:1} we derive
\begin{align*}
  \nabla f_{2,\theta_0}(\theta)
  &= \nabla \bigl[\nabla f(\theta_0)^T (\theta - \theta_0) \bigr]
  + \nabla \bigl[  \frac{1}{2} (\theta - \theta_0)^T \nabla^2 f(\theta_0) (\theta - \theta_0) \bigr]\\
  &= \nabla f(\theta_0) + \nabla^2 f(\theta_0)  (\theta - \theta_0) \;,
\end{align*}
which, when equated to $0$, yields
\begin{equation*}
  \hat{\theta} = \argmin_\theta f_{2,\theta_0}(\theta) =
  \theta_0 - \bigl[ \nabla^2 f(\theta_0) \bigr]^{-1} \nabla f(\theta_0) \;.
\end{equation*}
\end{proof}

\begin{proof}[Proof of Lemma~\ref{L:theta-diff}]
\label{proof:L:theta-diff}
By construction, $\nabla u_\epsilon(\hat\theta_{u_\epsilon})=0$ as $u_\epsilon$ assumes a minimum
at $\hat\theta_u$.
To prove~\eqref{E:0} from page~\pageref{L:theta-diff}, we first apply Lemma~\ref{L:1} to $u_\epsilon$ at
$\theta_0=\hat\theta$ leading to
\begin{equation}
\label{eq:theta_diff}
\hat\theta_u-\hat\theta=
  -
  \left[\nabla^2\left( f(\hat\theta)+\sum_t\epsilon_t L(t,\hat\theta)\right)\right]^{-1}
  \nabla \left[f(\hat\theta) + \sum_t\epsilon_t L(t,\hat\theta)\right]\;.
\end{equation}
Since $\nabla f(\hat\theta)=0$ and
$\nabla \bigl[ f(\hat\theta)+\sum_t\epsilon_t L(t,\hat\theta)\bigr] =
\sum_t \epsilon_t \nabla L(t,\hat\theta) $, for $\epsilon=0$ holds
\begin{equation}
  \label{eq:loss_lemma}
\frac{\partial}{\partial \epsilon_z}\hat\theta_{u_0} =
-\frac{\partial}{\partial\epsilon_z}\left[[\nabla^2 u_0(\hat\theta)]^{-1}\right]
\cdot \left[\sum_t\epsilon_t \nabla L(z,\hat\theta)\right]_{\epsilon=0}
-
[\nabla^2 u_0(\hat\theta)]^{-1} \cdot \nabla L(z,\hat\theta).
\end{equation}
Clearly, the first term on the right in~\eqref{eq:loss_lemma} vanishes
since its second factor is zero. For the second term we note that since $u_{0}(\theta) = f(\theta)$ and thus
have $\left[\nabla^2(u_{0})\right]^{-1}=
\left[\nabla^2 f\right]^{-1} =H_f^{-1} = H^{-1}$
which inserted into~\eqref{eq:loss_lemma} yields~\eqref{E:0}
\begin{equation*}
  \frac{\partial}{\partial\epsilon_z} \hat\theta_{u_{0}} =
- H^{-1} \nabla L(z,\hat\theta) \;.
\end{equation*}
\end{proof}

\begin{proof}[Proof of Lemma~\ref{L:basic}]
\label{proof:L:basic}
Applying Lemma~\ref{L:theta-diff} to $L(z',\hat\theta_{u})$ we derive
\begin{equation*}
  \frac{\partial}{\partial\epsilon_z}L(z',\hat\theta_{u_{0}})
   =  \left[ \nabla L(z',\hat\theta_{u_{0}}) \right]^T
  \frac{\partial}{\partial\epsilon_z}\hat\theta_{u_{0}}
 = -\nabla L(z',\hat\theta)^T  \left[\nabla^2 f(\hat\theta)\right]^{-1}\nabla L(z,\hat\theta) \;,
\end{equation*}
and since~$H$ and therefore $H^{-1}$ is symmetric,
\begin{eqnarray*}
\frac{\partial}{\partial\epsilon_z}L(z',\hat\theta_{u_{0}})
& = & -\Bigl[\nabla L(z',\hat\theta)^T  \left[\nabla^2 f(\hat\theta)\right]^{-1}\nabla L(z,\hat\theta) \Bigr]^T\\
& = & - \nabla L(z,\hat\theta)^T  \left[\nabla^2 f(\hat\theta)\right]^{-1}\nabla L(z',\hat\theta) \\
& = & -\left[\nabla L(z,\hat\theta_{u_{0}})\right]^T \frac{\partial}{\partial\epsilon_{z'}}\hat\theta_{u_{0}} \\
& = & \frac{\partial}{\partial\epsilon_{z'}}L(z,\hat\theta_{u_{0}}) \;,
\end{eqnarray*}
concluding the proof.
\end{proof}

\subsection{Proofs for main results}

\begin{proof}[Proof of Lemma~\ref{L:double-diff-theta}]
\label{proof:L:double-diff-theta}
  From the expression for $\hat\theta_u - \hat\theta$ from~\eqref{eq:theta_diff} and following the proof
  of Lemma~\ref{L:theta-diff} we employ Lemma~\ref{L:1} for $\theta_0=\hat\theta$ and $u(\theta)=f(\theta) +
  \sum_t\epsilon_t L(t,\theta)$ to obtain
  $ \hat\theta_{u}-\hat\theta=-
 \left[\nabla_\theta^2u\right]^{-1}
 \nabla_\theta u$. Since $\frac{\partial}{\partial \epsilon_z} \nabla u= \nabla L(z,\hat\theta)$ we derive
 \begin{eqnarray*}
 \frac{\partial}{\partial \epsilon_z}\hat\theta_{u} & = &  - \frac{\partial}{\partial\epsilon_z} \left[\nabla^2 u(\hat\theta)
   \right]^{-1} \cdot \nabla u  -
      [\nabla^2 u(\hat\theta)]^{-1} \cdot \nabla L(z,\hat\theta).
 \end{eqnarray*}
 Differentiating with respect to $\epsilon_x$ we have  $\frac{\partial}{\partial \epsilon_x}\nabla u=\nabla L(x,\hat\theta)$ and
\begin{eqnarray*}
  \frac{\partial^2}{\partial \epsilon_x \partial \epsilon_z}\hat\theta_u
  &=&
  -\frac{\partial^2}{\partial \epsilon_x \partial \epsilon_z}\left[ [\nabla^2 u(\hat\theta)]^{-1}\right]
  \cdot \nabla u
  -\frac{\partial}{\partial\epsilon_z}\left[ \nabla^2 u(\hat\theta)\right]^{-1}
  \cdot {\nabla L(x,\hat\theta)}  \\
 & &
  - \frac{\partial }{\partial \epsilon_x}\left[\nabla^2 u(\hat\theta)\right]^{-1} \cdot \nabla L(z,\hat\theta)
\end{eqnarray*}
Setting $\epsilon=0$ we have $\nabla u_0=0$ and accordingly
\begin{eqnarray*}
  \frac{\partial^2}{\partial \epsilon_x \partial \epsilon_z}\hat\theta_{u_0} &=& -\left[\frac{\partial}{\partial\epsilon_z}\left[
  \nabla_\theta^2 u_{  0}(\hat\theta)\right]^{-1} \cdot
\nabla L(x,\hat\theta)
+
\frac{\partial }{\partial \epsilon_x}\left[\nabla^2 u_{  0}(\hat\theta)\right]^{-1} \cdot \nabla L(z,\hat\theta)\right].
\end{eqnarray*}
Using the identity
$\frac{\partial}{\partial\epsilon_z}(A^{-1})=-A^{-1} \cdot
\frac{\partial}{\partial\epsilon_z}(A)\cdot A^{-1}$, and setting
$H = \nabla^2(u_{ 0}(\hat\theta))$ we obtain
\begin{align*}
  \frac{\partial}{\partial\epsilon_z} \left[\nabla^2 u_{  0} (\hat\theta)\right]^{-1}
  &=
  -H^{-1} \cdot \frac{\partial}{\partial\epsilon_z} \left[\nabla^2 f(\hat\theta) + \sum_t \epsilon_t \nabla^2 L(t,\hat\theta)
    \right] \cdot H^{-1} \\
  &=  -H^{-1}\cdot \left[\nabla^2 L(z,\hat\theta)\right]\cdot H^{-1}\;,\text{\quad and}\\
  \frac{\partial}{\partial\epsilon_x}\left[\nabla^2 u_{  0}(\hat\theta)\right]^{-1}
  &=  - H^{-1}\cdot \frac{\partial}{\partial\epsilon_x}\left[
        \nabla^2 f(\hat\theta) + \sum_t \epsilon_t \nabla^2 L(t,\hat\theta) \right] \cdot H^{-1} \\
  &= - H^{-1} \cdot \left[\nabla^2 L(x,\hat\theta)\right] \cdot H^{-1} \;,
\end{align*}
from which the proof follows.
\end{proof}

\begin{proof}[Proof of Lemma~\ref{L:double-diff}]
\label{proof:L:double-diff}
Using Lemma~\ref{L:theta-diff} we compute
\begin{eqnarray*}
  \frac{\partial }{\partial \epsilon_x}\frac{\partial}{\partial \epsilon_z}L(z',\hat\theta_u)
  & = & \frac{\partial}{\partial \epsilon_x}  \left[\left( \nabla L(z',\hat\theta_u) \right)^T  \frac{\partial}{\partial \epsilon_z} \hat\theta_u \right]\\
  & = & \left[ \frac{\partial}{\partial \epsilon_x}  \nabla L(z',\hat\theta_u) \right]^T \frac{\partial}{\partial \epsilon_z} \hat\theta_u  +
  \left(\nabla L(z',\hat\theta_u)\right)^T  \frac{\partial}{\partial \epsilon_x}\left[\frac{\partial}{\partial
      \epsilon_z} \hat\theta_u\right] \;,
\end{eqnarray*}
and, employing Lemmata~\ref{L:basic} and~\ref{L:double-diff-theta}, we obtain at $\epsilon=0$
\begin{eqnarray*}
\frac{\partial }{\partial \epsilon_x}\frac{\partial}{\partial \epsilon_z}L(z',\hat\theta_{u_{0}})
& = &
\left[ \frac{\partial}{\partial \epsilon_x}  \nabla L(z',\hat\theta_u) \right]^T \frac{\partial}{\partial \epsilon_z} \hat\theta_u  + \\
& &  \left(\nabla L(z',\hat\theta)\right)^T
\left[H_{u_{0}}^{-1} \left[\nabla^2 L(z,\hat\theta)\right] H_{u_{0}}^{-1}  \cdot \nabla L(x,\hat\theta)\right. \\
  & & \qquad \qquad \quad \ \  \left. + H_{u_{0}}^{-1} \left[\nabla^2 L(x,\hat\theta)\right] H_{u_{0}}^{-1}  \cdot
  \nabla L(z,\hat\theta)\right].
\end{eqnarray*}
Since
$
\frac{\partial }{\partial \epsilon_x}\frac{\partial}{\partial \epsilon_z}L(z',\hat\theta_u)
=\frac{\partial }{\partial \epsilon_z}\frac{\partial}{\partial \epsilon_x}L(z',\hat\theta_u)
$
we derive
 \begin{eqnarray*}
\left[ \frac{\partial}{\partial \epsilon_x}  \nabla L(z',\hat\theta_u) \right]^T \frac{\partial}{\partial \epsilon_z} \hat\theta_u  +
  \left(\nabla L(z',\hat\theta_u)\right)^T  \frac{\partial}{\partial \epsilon_x}\left[\frac{\partial}{\partial
      \epsilon_z} \hat\theta_u\right] & = & \\
  \left[ \frac{\partial}{\partial \epsilon_z}  \nabla L(z',\hat\theta_u) \right]^T \frac{\partial}{\partial \epsilon_x} \hat\theta_z  +
  \left(\nabla L(z',\hat\theta_u)\right)^T  \frac{\partial}{\partial \epsilon_z}\left[\frac{\partial}{\partial
      \epsilon_x} \hat\theta_u\right] & &
 \end{eqnarray*}
and using
$\frac{\partial}{\partial \epsilon_x}\left[\frac{\partial}{\partial \epsilon_z} \hat\theta_u\right]=
\frac{\partial}{\partial \epsilon_z}\left[\frac{\partial}{\partial \epsilon_x} \hat\theta_u\right]$ produces the equality
$$
\left[ \frac{\partial}{\partial \epsilon_x}  \nabla L(z',\hat\theta_u) \right]^T \frac{\partial}{\partial \epsilon_z} \hat\theta_u =
\left[ \frac{\partial}{\partial \epsilon_z}  \nabla L(z',\hat\theta_u) \right]^T \frac{\partial}{\partial \epsilon_x} \hat\theta_u.
$$
By Lemma~\ref{L:theta-diff} holds
$$
\left[ \frac{\partial}{\partial \epsilon_z}  \nabla L(z',\hat\theta_u) \right]^T \frac{\partial}{\partial \epsilon_x} \hat\theta_u =
-\big\langle \frac{\partial}{\partial \epsilon_z}  \nabla L(z',\hat\theta_u), \nabla L(x,\hat\theta_u)\big\rangle
$$
and the lemma follows.
\end{proof}
\begin{proof}[Proof of Lemma~\ref{L:claim0}]
\label{proof:L:claim0}
 We consider $\sum_{y\in Z} \frac{\partial^2 }{\partial
   \epsilon_y\partial \epsilon_z}L(z',\hat\theta_u)$ and interpret the
 terms with the help of Lemma~\ref{L:double-diff}. The lemma shows
 that each of these terms is in fact the sum of three values of the
 bilinear form, $\langle w,v \rangle = w^T H^{-1} v$, such that $y$
 only appears in the second variable. Using bilinearity, $\sum_{y\in
   Z}\nabla^2 L(y,\hat\theta_{u_{0}})=n H$ and $\nabla
 f_Z(\hat\theta_{u_0})=0$, we obtain for the sum of the first terms
\begin{align*}
  -\sum_{y\in Z}\big\langle  \frac{\partial}{\partial \epsilon_z}  \nabla L(z',\hat\theta_{u_0}) , L(y,\hat\theta_{u_0})\big\rangle
  &=
\big\langle  \frac{\partial}{\partial \epsilon_z}  \nabla L(z',\hat\theta_{u_0}) ,\sum_{y\in Z}\nabla L(y,\hat\theta_{u_{0}})\big\rangle \\
&=
n  \big\langle \frac{\partial}{\partial \epsilon_z}  \nabla L(z',\hat\theta_{u_0}) ,\nabla f_Z(\hat\theta_{u_{0}})
\rangle =0 \; .
\end{align*}
The sum of the second terms is given by
\begin{align*}
  \sum_{y\in Z}\langle \nabla L(z',\hat\theta_{u_{0}}),  \nabla^2 L(y,\hat\theta_{u_{0}})H^{-1}\nabla L(z,\hat\theta_{u_{0}})\rangle  &=
  \langle \nabla L(z',\hat\theta_{u_{0}}),  \sum_{y\in Z}\nabla^2 L(y,\hat\theta_{u_{0}}) \left(H^{-1}\nabla L(z,\hat\theta_{u_{0}})\right)\rangle \;.
\end{align*}
Here $\nabla^2 L(y,\hat\theta_{u_{0}})$ is an $m\times m$ matrix and
since $m\times m$ matrices form an algebra we can consider the
$m\times m$ matrix $\sum_{y\in Z}\nabla^2 L(y,\hat\theta_{u_{0}})$
where $(\sum_iA_i)x=\sum_i (A_ix)$. Accordingly we derive
\begin{align*}
 \sum_{y\in Z}\langle \nabla L(z',\hat\theta_{u_{0}}),  \nabla^2 L(y,\hat\theta_{u_{0}})H^{-1}\nabla L(z,\hat\theta_{u_{0}})\rangle  &=
  \langle \nabla L(z',\hat\theta_{u_{0}}),  \left[\sum_{y\in Z}\nabla^2 L(y,\hat\theta_{u_{0}})\right]H^{-1}\nabla L(z,\hat\theta_{u_{0}})\rangle \\
 &=   n \langle \nabla L(z',\hat\theta_{u_{0}}) , \nabla L(z,\hat\theta_{u_{0}})\rangle\;, \\
\end{align*}
since $HH^{-1}$ is the identity matrix. Finally we use linearity of
$\nabla^2 L(z,\hat\theta_{u_{0}}) H^{-1}$ and conclude
\begin{align*}
  \sum_{y\in Z} \langle \nabla L(z',\hat\theta_{u_{0}}), \nabla^2 L(z,\hat\theta_{u_{0}}) H^{-1} \nabla L(y,\hat\theta_{u_{0}})\rangle
  &=
\langle \nabla L(z',\hat\theta_{u_{0}}), \nabla^2 L(z,\hat\theta_{u_{0}}) H^{-1}\left[\sum_{y\in Z} \nabla L(y,\hat\theta_{u_{0}})\right]\rangle\\
&=
 n  \langle \nabla L(z',\hat\theta_{u_{0}}), \nabla^2 L(z,\hat\theta_{u_{0}}) H^{-1} \nabla f_Z(\hat\theta_{u_{0}})\rangle =0 \;.
\end{align*}
Accordingly we have
$$
\sum_{y\in Z}    \frac{\partial^2} {\partial \epsilon_y\partial \epsilon_z} L(z',\hat\theta_{u_{0}} )=
   n \frac{\partial}{\partial\epsilon_z} L(z',\hat\theta_{u_{0}}) \;,
$$
and Lemma~\ref{L:claim0} follows.
\end{proof}

\begin{proof}[Proof of Lemma~\ref{L:claim1}]
\label{proof:L:claim1}
  To establish eq.~(\ref{E:error}), instead of summing over all $y\in Z$, we introduce the term $O(n^{-2})$ and
  single out the two terms $y=z$ and $y=z'$
\begin{eqnarray*}
  \sum_{y\neq z,z'}    \left[\frac{\partial^2} {\partial \epsilon_y\partial \epsilon_z} L(z',\hat\theta_{u_{0}} ) \frac{1}{n}+O(n^{-2})\right]
  +  \left[ \frac{\partial^2} {\partial \epsilon_{z'}\partial \epsilon_z} L(z',\hat\theta_{u_{0}} )
  + \frac{\partial^2} {\partial \epsilon_{z}\partial \epsilon_z} L(z',\hat\theta_{u_{0}} )\right] \frac{1}{n}+O(n^{-2})
  = \\
\langle \nabla L(z',\hat\theta) ,\nabla L(z,\hat\theta_{u_{0}})\rangle +O(n^{-1}),
\end{eqnarray*}
since $\sum_{y\in Z}O(n^{-2})=O(n^{-1})$. We immediately observe
$$
 \left[ \frac{\partial^2} {\partial \epsilon_{z'}\partial \epsilon_z} L(z',\hat\theta_{u_{0}} )
  + \frac{\partial^2} {\partial \epsilon_{z}\partial \epsilon_z} L(z',\hat\theta_{u_{0}} )\right] \frac{1}{n}+O(n^{-2}) =O(n^{-1}) \;,
$$
and arrive at
\begin{align*}
   \sum_{y\neq z,z'}    \left[\frac{\partial^2} {\partial \epsilon_y\partial \epsilon_z} L(z',\hat\theta_{u_{0}} ) \frac{1}{n}+O(n^{-2})\right] &=
   \langle \nabla L(z',\hat\theta_{u_0}),  \nabla L(z,\hat\theta_{u_{0}})\rangle   +O(n^{-1}) \\
   &= \frac{\partial}{\partial\epsilon_z}L(z',\hat\theta_{u_{0}}) +O(n^{-1}) \;,
\end{align*}
whence Lemma~\ref{L:claim1}.
\end{proof}

\begin{proof}[Proof of Theorem~\ref{T:manyk}]
\label{proof:T:manyk}
We view
$\left[\frac{\partial}{\partial\epsilon_z}L(z',\hat\theta_{u})\right]^k$
as a function of $\epsilon = (\epsilon_t)_t$, form its Taylor
expansion at $(\epsilon_t)_t=0$, and subsequently evaluate at
$\epsilon = \eta_M$. This expansion involves terms of partial
derivatives $\frac{\partial}{\partial y}$ for $y\in M$ and $s\le n_0$
guarantess that the number of these terms does not scale with
$n$. Consequently we obtain
\begin{equation}\label{E:local44}
 \left[\frac{\partial}{\partial\epsilon_z}L(z',\hat\theta_{u_{  \eta_M}})\right]^k -\left[\frac{\partial}{\partial\epsilon_z}
   L(z',\hat\theta_{u_{  0}}) \right]^k = \sum_{y\in M} \left[\frac{\partial}{\partial \epsilon_y}
 \left[\frac{\partial}{\partial\epsilon_z}L(z',\hat\theta_{u_{  0}} )\right]^k
 \cdot
\left(-\frac{1}{n}\right)  + O(n^{-2})\right]
\end{equation}
Here $z',z$ are selected as distinct elements of $Z$ and $M\subset
Z\setminus \{z',z\}$ and the error term $O(n^{-2})$ is welldefined,
since $s\le n_0$ and $n$ is sufficiently large. We next consider pairs
$(M,\{z',z\})$ where $z,z'\not\in M$. A fixed $\{z',z\}$ appears in
such pairs with multiplicity $\binom{n-2}{s}$.  Consequently summing
over all pairs $(M,\{z',z\})$ and normalizing by
$\rho_s=\frac{1}{\binom{n}{2}\binom{n-2}{s}}$ we derive
$$
\frac{1}{\binom{n-2}{s}} \frac{1}{\binom{n}{2}}
\sum_{(M,\{z',z\}) \atop \vert M\vert=s} \left[\frac{\partial}{\partial\epsilon_z} L(z',\hat\theta_{u_{  0}}) \right]^k =\mathbb{E}[X_{u_0}^k].
$$
Note that $\binom{n}{2}\binom{n-2}{s}=\binom{n-s}{2}\binom{n}{s}$,
where for fixed $M$, the coefficient $\binom{n-s}{2}$ counts how many
pairs $\{z,z'\}$ not contained in $M$ can be selected. Therefore we
can conclude
$$
\frac{1}{\binom{n-s}{2}} \sum_{\{z,z'\}\cap M=\varnothing}  \left[\frac{\partial}{\partial\epsilon_z}
    L(z',\hat\theta_{u_{ \eta_M}})\right]^k = \mathbb{E}[X_{u_{\eta_M}}^k],
$$
with associated loss function $f_{Z\setminus
  M}(\theta)=\frac{1}{n-s}\sum_{t\in Z\setminus M} L(t,\theta)$, and
\begin{eqnarray*}
\frac{1}{\binom{n}{s}} \frac{1}{\binom{n-s}{2}}  \sum_{(M,\{z',z\}) \atop \vert M\vert =s}  \left[\frac{\partial}{\partial\epsilon_z}
  L(z',\hat\theta_{u_{ \eta_M}})\right]^k & = & \frac{1}{\binom{n}{s}} \sum_{M\subset  Z\atop \vert M\vert =s}
\left[\frac{1}{\binom{n-s}{2}} \sum_{\{z,z'\}\cap M=\varnothing}  \left[\frac{\partial}{\partial\epsilon_z}
    L(z',\hat\theta_{u_{ \eta_M}})\right]^k\right] \\
& = & \frac{1}{\binom{n}{s}}\sum_{M\subset  Z\atop \vert M\vert =s}  \mathbb{E}[X_{u_{\eta_M}}^k].
\end{eqnarray*}
Since $f_{Z\setminus M}$ assumes a minimum at $\hat\theta_{u_{\eta_M}}$ we derive from~\eqref{E:local44}
\begin{align*}
  \Delta_{s}(\mathbb{E}[X_{u_0}^k])
  &=
  \rho_s \sum_{(M,\{z',z\})} \left[\frac{\partial}{\partial\epsilon_z} L(z',\hat\theta_{u_{0}}) \right]^k -
  \left[\frac{\partial}{\partial\epsilon_z}L(z',\hat\theta_{u_{\eta_M}})\right]^k \\
  &= \rho_s\sum_{(M,\{z',z\})} \left[ \sum_{y\in M} \left[\frac{\partial}{\partial \epsilon_y}
    \left[\left[\frac{\partial}{\partial \epsilon_z} L(z',\hat\theta_{u_{0}} )\right]^k \right]
    \cdot \left(\frac{1}{n}\right)  + O(n^{-2})\right]\right]\\
  &=   k \rho_s\sum_{(M,\{z',z\})} \left[\left[ \frac{\partial}{\partial\epsilon_z}
      L(z',\hat\theta_{u_{0}})\right]^{k-1}  \sum_{y\in M}\left[
      \frac{\partial^2}
       {\partial \epsilon_y\partial \epsilon_z} L(z',\hat\theta_{u_{0}} ) \cdot
    \left(\frac{1}{n}\right)  + O(n^{-2})\right  ]\right]\\
&= k \rho_s \sum_{\{z,z'\}} \left[\left[ \frac{\partial}{\partial\epsilon_z}
    L(z',\hat\theta_{u_{0}}) \right]^{k-1} \left[\sum_{M} \sum_{y\in M}\left[\frac{\partial^2}
       {\partial \epsilon_y\partial \epsilon_z} L(z',\hat\theta_{u_{0}} ) \cdot
    \left(\frac{1}{n}\right)  + O(n^{-2})\right]\right ]\right]  .
\end{align*}
Any given $y\in Z\setminus \{z,z'\}$ appears in pairs $(M,\{z,z'\})$
with multiplicity $\binom{n-3}{s-1}$ and consequently
\begin{equation}\label{E:reorganize}
\sum_{M} \sum_{y\in M}\left[\frac{\partial^2}
       {\partial \epsilon_y\partial \epsilon_z} L(z',\hat\theta_{u_{0}} ) \cdot
       \left(\frac{1}{n}\right)  + O(n^{-2})\right] =
       \binom{n-3}{s-1}
       \sum_{y\in Z\atop y\neq z,z'}   \left[ \frac{\partial^2} {\partial \epsilon_y\partial \epsilon_z} L(z',\hat\theta_{u_{0}} )
         \left(\frac{1}{n}\right)+O(n^{-2})\right]
\end{equation}
Lemma~\ref{L:claim1} allows us to conclude
\begin{align*}
k \rho_s \sum_{\{z,z'\}, z\neq z'\atop z,z'\in Z}& \left[ \left[\frac{\partial}{\partial\epsilon_z}
  L(z',\hat\theta_{u_{0}}) \right]^{k-1} \left[
    \binom{n-3}{s-1} \sum_{y\in Z \setminus \{z,z'\}} \left[\frac{\partial^2}
       {\partial \epsilon_y\partial \epsilon_z} L(z',\hat\theta_{u_{0}} ) \cdot
       \left(\frac{1}{n}\right)  + O(n^{-2})\right]\right ]\right] \\
&=
k\rho_s \sum_{\{z,z'\}, z\neq z'\atop z,z'\in Z}  \left[ \left[\frac{\partial}{\partial\epsilon_z}
  L(z',\hat\theta_{u_{0}}) \right]^{k-1} \left[
    \binom{n-3}{s-1}   \left[\frac{\partial}{\partial\epsilon_z}L(z',\hat\theta_{u_{0}}) +O(n^{-1})   \right]\right]\right] \\
&=
k \frac{\binom{n-3}{s-1}}{\binom{n}{2}\binom{n-2}{s}}\sum_{\{z,z'\}, z\neq z'\atop z,z'\in Z}  \left[\left[ \frac{\partial}{\partial\epsilon_z}
  L(z',\hat\theta_{u_{0}})  \right]^k +O(n^{-1})\right] \;.
\end{align*}
Since  $\frac{\binom{n-3}{s-1}}{\binom{n-2}{s}}=\frac{s}{n-2}$, $\rho_s^{-1}=\binom{n}{2}\binom{n-2}{s}$,
and
$$
\frac{1}{\binom{n}{2}}\sum_{\{z,z'\}, z\neq z'\atop z,z'\in Z}
\left[ \frac{\partial}{\partial\epsilon_z} L(z',\hat\theta_{u_{0}})  \right]^k
= \mathbb{E}[X_{u_0}^k],
$$
we arrive at
\begin{align*}
   \Delta_s(\mathbb{E}[X_{u_0}^k] &=
  \mathbb{E}[X_{u_0}^k] - \frac{1}{\binom{n}{s}} \sum_{M\subset  Z\atop \vert M\vert =s}
  \mathbb{E}[X_{u_{\eta_M}}^k]  =   \frac{k s}{n-2}  \mathbb{E}[X_{u_0}^k]   +   \frac{k s}{n-2}  O(n^{-1+2-2})\;,
\end{align*}
completing the proof of the theorem.
\end{proof}

\begin{proof}[Proof of Theorem~\ref{T:variance}]
\label{proof:T:variance}
  We denote $X=X_{u_0}$ and $X_M=X_{u_{\eta_M}}$ and compute using Theorem~\ref{T:manyk} for $k=2$
\begin{eqnarray*}
\mathbb{V}[X]  - \frac{1}{\binom{n}{s}} \sum_{M\subset  Z\atop \vert M\vert = s} \mathbb{V}[X_M] & = &
\left[ \mathbb{E}[X^2]-\mathbb{E}[X]^2\right]  -
\frac{1}{\binom{n}{s}} \sum_{M\subset  Z\atop \vert M\vert = s}
\left[ \mathbb{E}[X_M^2]-\mathbb{E}[X_M]^2\right] \\
& = & \left[\mathbb{E}[X^2] - \frac{1}{\binom{n}{s}} \sum_{M\subset  Z\atop \vert M\vert = s} \mathbb{E}[X_M^2]\right]
-\left[\mathbb{E}[X]^2 -\frac{1}{\binom{n}{s}} \sum_{M\subset  Z\atop \vert M\vert = s} \mathbb{E}[X_M]^2\right] \\
& = &
\left[\frac{2s}{n-2}\mathbb{E}[X^2]  +O(n^{-2}) \right] -
\left[\mathbb{E}[X]^2 -\frac{1}{\binom{n}{s}} \sum_{M\subset  Z\atop \vert M\vert = s} \mathbb{E}[X_M]^2\right]
\end{eqnarray*}
It remains to compute the second term: we claim that for any $M$ holds $\mathbb{E}[X_M]=\mathbb{E}[X]=O(n^{-1})$.
This follows from $\nabla f_{Z}(\hat\theta_{u_0})=\nabla f_{Z\setminus M}(\hat\theta_{u_{\eta_M}})=0$,
and Lemma~\ref{L:basic}. For $\hat\theta = \hat\theta_{u_{\eta_M}}$, $\vert M\vert =s$ we have
$$
\sum_{z,z'\in Z\setminus M}\frac{\partial }{\partial\epsilon_z}L(z',\hat\theta)=
\sum_{z'\in Z\setminus M}\nabla L(z',\hat\theta)^T  H^{-1}\left(\sum_{z\in Z\setminus M}\nabla L(z,\hat\theta)\right)= 0.
$$
and setting $M=\varnothing$, $s=0$ we obtain an analogous expression for $X_{u_0}$.
This in turn implies $\mathbb{E}[X_{u_0}]=\mathbb{E}[X_{u_{\eta_M}}]=O(n^{-1})$ and consequently,
\begin{eqnarray*}
\mathbb{E}[X]^2 -\frac{1}{\binom{n}{s}} \sum_{M\subset  Z\atop \vert M\vert = s} \mathbb{E}[X_M]^2
 & = &O(n^{-2}),
\end{eqnarray*}
establishing
\begin{eqnarray*}
  \mathbb{V}[X]  - \frac{1}{\binom{n}{s}} \sum_{M\subset  Z\atop \vert M\vert = s} \mathbb{V}[X_M] & = &
  \left[\frac{2s}{n-2}\mathbb{E}[X^2]  +O(n^{-2}) \right]  + O(n^{-2})\\
  & = & \frac{2s}{n-2}\left(\mathbb{E}[X^2] -\mathbb{E}[X]^2\right) + O(n^{-2}) \;.
  \end{eqnarray*}
\end{proof}

\begin{proof}[Proof of Corollary~\ref{C:1}]
\label{proof:C:1}
We prove the corollary by contradiction, suppose~\eqref{E:exists} does
not hold. Then for any $M\subset Z$ with $\vert M\vert=s$ holds
 $$
 \frac{ 2 s}{n-2} \mathbb{E}[X^k_{u_0}] + O(n^{-2}) >
 \mathbb{E}[X^k_{u_0}]  - \mathbb{E}[X^k_{u_{\eta_M}}] .
 $$
Summing over all such $M\subset Z$ and normalizing by
$\frac{1}{\binom{n}{s}}$ has no effect on any terms that are
independent of $M$, whence we arrive at
  $$
  \frac{ 2 s}{n-2}  \mathbb{E}[X^k_{u_0}]  + O(n^{-2})
  >
  \mathbb{E}[X^k_{u_0}] - \frac{1}{\binom{n}{s}} \sum_{Z\subset M\atop \vert Z\vert=s} \mathbb{E}[X^k_{u_{\eta_M}}]
  $$
which is in view of Theorem~\ref{T:manyk}, impossible.
\end{proof}


\section*{Data sets and source code}

The source code used for the experiments in this paper is available as
part of the Git repository:
\begin{equation*}
    \text{\url{https://github.com/FenixHuang667/DataPurification}}
\end{equation*}
Python v.3.12.3 was used with PyTorch v.2.9.0. The data displayed in
Figures~3--8 through Section~\ref{S:appl} is included in the Git
repository.


\section*{Acknowledgements}
Funding provided by the UVA National Security \& Data Institute,
Contracting Activity \# 2024-24070100001. We thank Phillip Potter for
discussions and feedback on the manuscript, we gratefully acknowledge
the help of Jingyi Gao for preparing the EMNIST data.

\bibliographystyle{plain}
\bibliography{main}
\end{document}